\documentclass{article}

\usepackage{times}  %
\usepackage{helvet}  %
\usepackage{courier}  %
\usepackage[hyphens]{url}  %
\usepackage{graphicx} %
\usepackage{natbib}  %
\usepackage{caption} %
\usepackage{algorithm}
\usepackage[noend]{algorithmic}

\newcommand\algorithmicprocedure{\textbf{Procedure}}
\newcommand{\algorithmicendprocedure}{\algorithmicend\ \algorithmicprocedure}
\makeatletter
\newcommand\PROCEDURE[3][default]{%
	\ALC@it
	\algorithmicprocedure\ \textsc{#2}(#3)%
	\ALC@com{#1}%
	\begin{ALC@prc}%
	}
	\newcommand\ENDPROCEDURE{%
	\end{ALC@prc}%
	\ifthenelse{\boolean{ALC@noend}}{}{%
		\ALC@it\algorithmicendprocedure
	}%
}
\newenvironment{ALC@prc}{\begin{ALC@g}}{\end{ALC@g}}
\makeatother

\usepackage{newfloat}
\usepackage{listings}
\DeclareCaptionStyle{ruled}{labelfont=normalfont,labelsep=colon,strut=off} %
\floatstyle{ruled}
\newfloat{listing}{tb}{lst}{}
\floatname{listing}{Listing}
\usepackage{dsfont}
\usepackage{amsmath}
\usepackage{stmaryrd}
\usepackage{amssymb}
\usepackage{amsthm}
\usepackage{graphicx}
\usepackage{subcaption}
\usepackage{tikz}
\usepackage{forest}
\usetikzlibrary{automata, graphs,positioning,chains,arrows,snakes,decorations.pathmorphing}
\usetikzlibrary{backgrounds}

\usepackage
[disable]
{todonotes}

\newcommand{\martin}[1]{\todo[inline, color=green!20]{{\bf M:} #1}}
\newcommand{\cristian}[1]{\todo[inline, color=red!20]{{\bf C:} #1}}

\newcommand{\rodrigo}[1]{\todo[inline, color=yellow!20]{{\bf R:} #1}}

\newcommand{\true}{\textbf{true}}
\newcommand{\false}{\textbf{false}}
\newcommand{\galgo}{\Lambda}
\newcommand{\domins}{\mathcal{D}}
\newcommand{\sem}[1]{{\llbracket{}{#1}\rrbracket}}
\newcommand{\IneqF}{\mathcal{F}}

\newcommand{\cA}{\mathcal{A}}
\newcommand{\cL}{\mathcal{L}}

\newcommand{\cR}{\mathcal{R}}

\newcommand{\cC}{\mathcal{C}}

\newcommand{\bbN}{\mathbb{N}}

\newcommand{\bbZ}{\mathbb{Z}}

\newcommand{\bbQ}{\mathbb{Q}}
\newcommand{\bbQinf}{\mathbb{Q}_{\infty}}

\newcommand{\size}[1]{\operatorname{size}(#1)}

\newcommand{\bigo}{O}

\newcommand{\rep}{\mathfrak{r}}

\newcommand{\membOracle}[1]{\operatorname{MQ}(#1)}

\newcommand{\membOracleName}{\operatorname{MQ}}
\newcommand{\equivOracleName}{\operatorname{EQ}}

\newcommand{\SBtree}{T_{\operatorname{SB}}}

\newcommand{\bounds}{\operatorname{bounds}}
\newcommand{\unorderInt}[1]{\lfloor#1\rceil}

\newcommand{\lbound}[1]{\operatorname{lbound}(#1)}
\newcommand{\rbound}[1]{\operatorname{rbound}(#1)}
\newcommand{\leftChild}[1]{\operatorname{left}\!\left(#1\right)}
\newcommand{\rightChild}[1]{\operatorname{right}\!\left(#1\right)}
\newcommand{\parent}[1]{\operatorname{parent}\left(#1\right)}
\newcommand{\SBenc}[1]{\texttt{enc}_{\operatorname{SB}}\left(#1\right)}

\newcommand{\constructRepresentationName}{\operatorname{ConstructRepresentation}}
\newcommand{\constructRepresentation}[1]{\operatorname{ConstructRepresentation}(#1)}

\newcommand{\findClosestAncestor}[2]{\operatorname{FindClosestAncestor}(#1, #2)}

\newcommand{\findBreakLinkName}{\operatorname{FindBreakLink}}
\newcommand{\findBreakLink}[3]{\operatorname{FindBreakLink}(#1, #2, #3)}

\newcommand{\searchLeftName}{\operatorname{SearchLeft}}
\newcommand{\searchLeft}[3]{\operatorname{SearchLeft}(#1, #2, #3)}
\newcommand{\searchRightName}{\operatorname{SearchRight}}
\newcommand{\searchRight}[3]{\operatorname{SearchRight}(#1, #2, #3)}
\newcommand{\bsLeftName}{\operatorname{BSLeft}}
\newcommand{\bsLeft}[4]{\operatorname{BSLeft}(#1, #2, #3, #4)}
\newcommand{\bsRightName}{\operatorname{BSRight}}
\newcommand{\bsRight}[4]{\operatorname{BSRight}(#1, #2, #3, #4)}

\newtheorem{theorem}{Theorem}
\newtheorem{lemma}{Lemma}

\newtheorem{corollary}{Corollary}
\newtheorem{proposition}{Proposition}

\title{Active Learning of Symbolic Automata Over Rational Numbers}
\author{
    Sebastian Hagedorn\textsuperscript{\rm 1},
    Martín Muñoz\textsuperscript{\rm 1, \rm 2, \rm 3},
    Cristián Riveros\textsuperscript{\rm 1, \rm 2},\\
    Rodrigo Toro Icarte\textsuperscript{\rm 1, \rm 4}
}

\begin{document}
	
	\maketitle
	
\begin{center}
\small
    \textsuperscript{\rm 1}Department of Computer Science, Pontificia Universidad Católica de Chile, Santiago, Chile\\
    \textsuperscript{\rm 2}Instituto Milenio Fundamentos de los Datos IMFD, Santiago, Chile\\
    \textsuperscript{\rm 3}Univ. Artois, CNRS, UMR 8188, Centre de Recherche en Informatique de Lens
(CRIL), F-62300 Lens, France\\
    \textsuperscript{\rm 4}Centro Nacional de Inteligencia Artificial CENIA, Santiago, Chile
\end{center}

\begin{abstract}
Automata learning has many applications in artificial intelligence and software engineering. Central to these applications is the $L^*$ algorithm, introduced by Angluin \cite{angluin-learning-L*}. The $L^*$ algorithm learns deterministic finite-state automata (DFAs) in polynomial time when provided with a minimally adequate teacher. Unfortunately, the $L^*$ algorithm can only learn DFAs over finite alphabets, which limits its applicability. In this paper, we extend $L^*$ to learn symbolic automata whose transitions use predicates over rational numbers, i.e., over infinite and dense alphabets.
Our result makes the $L^*$ algorithm applicable to new settings like (real) RGX, and time series. 
Furthermore, our proposed algorithm is optimal in the sense that it asks a number of queries to the teacher that is at most linear with respect to the number of transitions, and to the representation size of the predicates. 
\end{abstract}

	\section{Introduction}\label{sec:introduction}

Automata learning has powered numerous advances in artificial intelligence and software engineering. Its applications range from interpretable sequence classification \cite[e.g.,][]{shvo2021interpretable, katzouris2022learning, roy2023learning} to reinforcement learning \cite[e.g.,][]{hasanbeig2021deepsynth, corazza2022reinforcement, toro2023learning}. Key to these advances is the $L^*$ algorithm \cite{angluin-learning-L*}.

Learning the smallest automaton consistent with a fixed set of examples is a well-known NP-hard problem \cite{GOLD1967447, ANGLUIN1980117}. However, Angluin \cite{angluin-learning-L*} showed that minimal automata can be learned in polynomial time under an \emph{active learning} framework. In this setting, the learner interacts with a \emph{minimally adequate teacher (MAT)}, which can answer two types of queries: \emph{membership queries}, which return the correct classification of a given string, and \emph{equivalence queries}, which verify whether a proposed automaton correctly represents the target language. If the hypothesis is incorrect, the teacher provides a counterexample. Under these assumptions, the $L^*$ algorithm can efficiently learn a minimal DFA in polynomial time, assuming constant-time responses to both query types.

The $L^*$ algorithm has had a significant and lasting influence on the field of automata learning, shaping both theoretical developments and practical applications \cite{vaandrager2022new}. However, it is limited to learning \emph{deterministic finite automata (DFA)}, which can only capture temporal behaviors over finite alphabets. This limitation restricts their use in domains involving infinite or dense input spaces, such as complex event recognition \cite{giatrakos2020complex} and time series learning \cite{hamilton2020time}.%

To address these limitations, van Noord and Gerdemann \cite{sfa-van-noord} and and Veanes et al. \cite{veanes2012symbolic} introduced the concept of symbolic automata, which generalize classical automata by allowing transitions to be labeled with predicates over rich alphabet theories, while preserving many of the desirable properties of traditional automata. The learnability and practical applications of symbolic automata have been extensively explored \cite{learning-symbolic-automata, Mens_2015, 10.1145/138027.138042}. However, existing approaches for learning symbolic automata assume that the MAT provides lexicographically minimal counterexamples. Without this assumption, it remains an open question whether such automata can be learned efficiently.

In this paper, we extend the $L^*$ algorithm to enable the learning of symbolic automata over the rational numbers. Our approach can learn any symbolic automaton whose predicates operate over a single rational-valued variable, without imposing any constraints on the form of counterexamples provided by the teacher. Furthermore, our method is query-efficient, requiring only a polynomial number of queries relative to the size of the target automaton. To achieve this, we first introduce a technique for learning \emph{finite piecewise functions} over the rational numbers using a MAT. This is accomplished through a binary search procedure over a number-theoretic structure known as the Stern–Brocot tree \cite{Stern1858,Brocot1861}. Interestingly, our learning algorithm is efficient in the sense that it needs a linear number of queries concerning the function size and does not depend on the worst-case size counterexample. We then integrate this interval learning strategy into the $L^*$ framework to support the learning of symbolic automata over rational numbers. Finally, we formally prove that our method can learn any such automaton in polynomial time under standard assumptions.

	\section{Preliminaries}\label{sec:learning-setting}

We begin by recalling the learning setting that we will use throughout the paper. We present the general definition here to later instantiate it for finite piecewise functions in Section~\ref{sec:setting} and for symbolic automata in Section~\ref{sec:learning-symbolic-automata}.

\paragraph{Learning domain.} Let $\domins$ be an infinite domain (e.g., sequences) and $\Sigma$ a finite domain (e.g., $\{0,1\}$). A \emph{concept} over $\domins$ is a function $\gamma: \domins \mapsto \Sigma$. A representation $\rep$ over $\domins$ is a finite object (i.e., given by some syntax and semantics) that encodes a function $\rep: \domins \mapsto \Sigma$. We denote by $\size{\rep}$ the size of $\rep$, namely, the number of bits to store $\rep$. We say that $\rep$ is a representation of a concept $\gamma$, denoted by $\rep = \gamma$, if $\rep(d) = \gamma(d)$ for every $d \in \domins$. We usually use $\cC$ and $\cR$ to denote a \emph{class of concepts} and a \emph{class of representations}, respectively. 
We assume that every concept $\cC$ has some representation in $\cR$, namely, for every $\gamma \in \cC$ there exists $\rep \in \cR$ such that $\rep = \gamma$.
Finally, we denote by $\size{\gamma} = \min_{\rep: \rep = \gamma} \size{\rep}$ the size of the minimum representation in $\cR$ of $\gamma \in \cC$.
\cristian{TODO: utilizer desde ahora los comandos $\true$ y $\false$ para los valores true y false.}

\paragraph{The MAT learning setting.}  Let $\cC$ and $\cR$ be a class of concepts and representations, respectively, over $\domins$. 
In general, given some concept $\gamma \in \cC$, the learning task consists of finding some $\rep \in \cR$ such that $\rep = \gamma$. For this task, we follow the active learning setting introduced by \cite{angluin-learning-L*} called \emph{Minimally Adequate Teacher} (MAT): 
the Learner has to infer the behavior of an unknown concept $\gamma\in \cC$ by having access to two functions, called \emph{membership} and \emph{equivalence oracles}, provided by the Teacher who knows $\gamma$ beforehand.
Specifically, a \emph{membership oracle} $\membOracleName$ of $\gamma$ answers the following query: 
given an input $d \in \domins$, return the output $\gamma(d)$. An \emph{equivalence oracle} $\equivOracleName$ of $\gamma$ receives a representation $\rep \in \cR$ and returns either the pair $(\true, \bot)$, meaning that $\rep = \gamma$, or the pair $(\false, d^*)$ where $d^*\in \domins$, meaning that $\rep \not= \gamma$ and $d^*$ is a \emph{counterexample} (i.e., $\rep(d^*) \neq \gamma(d^*)$). 
For any pair of functions $f,g:\bbN \rightarrow \bbN$, we say that an algorithm \emph{$\galgo$ learns $\cR$ using a MAT with $f(n)$ membership queries and $g(n)$ equivalence queries} if, given a membership oracle $\membOracleName$ and an equivalence oracle $\equivOracleName$ of some unknown concept $\gamma \in \cC$, $\galgo$ finds a representation $\rep \in \cR$ such that $\rep = \gamma$ after making at most $\bigo(f(n))$ membership queries and $\bigo(g(n))$ equivalence queries with $\size{\gamma} = n$.
As it is common in the literature, for measuring the efficiency of $\galgo$, we are interested in bounding the number of calls that $\galgo$ makes to the oracles with respect to $\size{\gamma}$ and, also, to the length of the longest counterexample provided by~$\equivOracleName$.

	\section{Learning finite piecewise functions}\label{sec:setting}

We address the problem of learning finite piecewise functions over the rational numbers. This learning problem presents a fundamental challenge: the learner must not only generalize from observed examples, but also precisely identify the separation boundaries that determine membership of an example. In contrast to discrete domains, where classification changes occur at clearly defined points, the density and order structure of the rational numbers require the learner to infer the exact transition thresholds at which the classification shifts to a different value.
Here, we present a foundational result at the core of this problem: how to efficiently classify a rational into one of multiple class labels.

We start by introducing some basic notation regarding rationals and intervals that we will use throughout the paper to then define the learning problem. We end this section by presenting the main technical result of the paper.

\paragraph{Rationals and intervals} Let $\bbN = \{0,1,\ldots\}$ be the set of natural numbers, $\bbZ$ the set of integers, $\bbQ$ the set of rational numbers, and $\bbQinf = \bbQ \cup \{\infty, -\infty\}$ where $-\infty < q < \infty$ for every $q \in \bbQ$.
We will use $a,b,c,\ldots$ to denote elements in $\bbN$ or $\bbZ$, and $p,q,r, \ldots$ for elements in $\bbQ$.
Further, we represent any $q \in \bbQinf$ as a {\em fraction} $q = \frac{a}{b}$ where $a \in \bbZ$ and $b \in \bbN$. We note that this includes the fractions $\frac{a}{0} = \infty$  and $\frac{-a}{0}= -\infty$ for any $a\in\bbN$. We say that $\frac{a}{b}$ is {\em irreducible} if there is no natural number $c > 1$ such that $c$ divides both $a$ and $b$.
In the sequel, we will always assume that $\frac{a}{b}$ is irreducible, where $\infty = \frac{1}{0}$ and $-\infty = \frac{-1}{0}$.
For $q = \frac{a}{b}\in\bbQ$ with $a \neq 0$, we define $\size{q} = \log(|a|) + \log(b)$ where $\log(\cdot)$ is in base~$2$, namely, the number of bits to represent $q$ as a fraction; and when $a = 0$, $\size{q} = 1$.

A subset $I \subseteq \bbQ$ is an \emph{interval} if $I$ is non-empty, and for every $p, q \in I$ and $r \in \bbQ$, if $p < r < q$, then $r \in I$.
As usual, we represent intervals as pairs $\langle p, q \rangle$ where $\langle$ could be a left-bracket $[$ (i.e., left-closed) or a left-parenthesis $($ (i.e., left-open), and $\rangle$ could be a right-bracket $]$ (i.e., right-closed) or a right-parenthesis $)$ (i.e., right-open). In addition, when $p= -\infty$ then $\langle\, = ($, and when $q = \infty$, then $\rangle = \, )$ (i.e., $-\infty$ or $\infty$ are never included). Further, we always write these $p$ and $q$ as irreducible fractions in $\bbQinf$. For example:
\[
\begin{array}{rcl}
	(\frac{-1}{2}, \frac{3}{1}] & = & \{p \in \bbQ \mid -\frac{1}{2} < p \leq 3\} \text{ and } \vspace{1mm}\\
	\left[\frac{7}{1}, \frac{1}{0}\right) & = & \{ p \in \bbQ \mid 7 \leq p < \infty \} 
\end{array}
\]
are intervals over $\bbQ$ under this notation. 
For $I = \langle p, q \rangle$ we define $\size{I} = \size{p} + \size{q}$.

For two intervals $I_1, I_2$, we write $I_1 < I_2$ iff $p < q$ for every $p \in I_1$ and $q \in I_2$.
We say that a sequence of intervals $I_1 I_2 \ldots I_k$ is an \emph{interval partition} of $\bbQ$ iff $I_1 < I_2 < \ldots < I_k$ and $\bbQ = \bigcup_{j=1}^k I_j$.
For instance,
\[
\big(\tfrac{-1}{0}, \tfrac{-2}{3}\big)\, \big[\tfrac{-2}{3}, \tfrac{1}{2}\big]\, \big(\tfrac{1}{2}, \tfrac{3}{2}\big]\, \big(\tfrac{3}{2}, \tfrac{1}{0}\big) 
\]
is an interval partition of $\bbQ$.

\paragraph{Finite piecewise functions.} Let $\Sigma$ be a finite set of labels. We will use $A,B,\ldots$ for denoting labels in $\Sigma$. 
The class $\cC$ of \emph{concepts} are all functions over the rationals $\gamma: \bbQ \mapsto \Sigma$. Further, we restrict ourselves to the class of \emph{finite piecewise functions}, namely, functions $\gamma$ that can be represented by a finite number of intervals.
Formally, we say that an interval $I$ is \emph{monochromatic} with respect to $\gamma$ if $\gamma(p) = \gamma(q)$ for every pair $p, q \in I$.
If $I$ is monochromatic with respect to $\gamma$, we define $\gamma(I) = \gamma(q)$ where $q$ is any element in  $I$. 
As an example, consider a function $\gamma_1:\bbQ\to\{A,B\}$:
\[
\gamma_1(q) =
\begin{cases}
A, & \text{ if } \tfrac{-2}{3} \leq q \leq \tfrac{1}{2} \text{ or } \tfrac{3}{2} < q,\\
B, & \text{ otherwise.}
\end{cases}
\]
One can check that interval $(\tfrac{1}{2}, \tfrac{3}{2}]$ is monochromatic with respect to $\gamma_1$, and that $\gamma_1((\tfrac{1}{2}, \tfrac{3}{2}]) = B$.

\cristian{NOTAR el cambio de nombre a "finite piecewise functions" que creo que es más correcto. TODO: Actualizar el resto del paper.}

We say that $\gamma$ is a \emph{finite piecewise function} if there exists an interval partition $I_1 I_2 \ldots I_k$ of $\bbQ$ such that each $I_j$ is monochromatic with respect to $\gamma$. If this is the case, we can represent $\gamma$ as the sequence of pairs:
\[
\rep = (I_1, \gamma(I_1)) \,(I_2, \gamma(I_2)) \,\ldots\, (I_k, \gamma(I_k))\,. \tag{$\dagger$}
\]
We will write $\rep = \gamma$ to denote that $\rep$ represents $\gamma$ and write $\rep(q)$ to denote the label $\gamma(q)$ for every $q \in \bbQ$. 
We define the class $\cR$ of all \emph{representations} $\rep$ like ($\dagger$) of finite piecewise functions. 
We define the size of $\rep$ as $\size{\rep} = \sum_{j} \size{I_j}$. 

Note that a representation of a finite piecewise function $\gamma$ is not unique. For example, $((\frac{-1}{0}, \frac{1}{1}), A) ([\frac{1}{1}, \frac{1}{0}), A)$ and $((\frac{-1}{0},\frac{1}{0}), A)$ represent the same concept $\gamma$ where $\gamma(q) = A$ for every $q \in \bbQ$. For this reason, we say that a representation like $(\dagger)$ is the \emph{canonical representation} of $\gamma$ if, in addition, $\gamma(I_j) \neq \gamma(I_{j+1})$ for every $j < k$. One can easily check that every concept $\gamma$ has a unique canonical representation. Finally, for any concept $\gamma$ with canonical representation $\rep$ we define $\size{\gamma} = \size{\rep}$. 
Continuing the example, $\gamma_1$ can be seen to be a finite piecewise function. Its canonical representation is shown in Figure~\ref{fig:stern-brocot-tree}.

\paragraph{Main technical result.}
Herein lies the problem of learning an unknown finite piecewise function $\gamma \in \cC$ that can be finitely represented by some $\rep \in \cR$. In this work, we present an MAT learning algorithm that does this efficiently, focusing on finding each labeled interval $(I, \gamma(I))$ in time $\bigo(\size{I})$.  
Specifically, we get the following technical result on learning finite piecewise functions.
\begin{theorem}\label{theo:learning-intervals}
	For the class $\cC$ of finite piecewise functions, there exists a learning algorithm that learns an unknown concept $\gamma \in \cC$ using MAT with a linear number of membership and equivalence queries over $\size{\gamma}$.
\end{theorem}
In other words, no matter the concept $\gamma$, the number of calls to the oracles is proportional to the size of the canonical representation of $\gamma$, where rational numbers are represented in a binary encoding.
Arguably, this MAT learning algorithm is 
\martin{efficient $\to$ optimal?}
optimal in the sense that one cannot take less than $\size{\gamma}$. Furthermore, the learning guarantee does not depend on the size of the longest counterexample as other MAT algorithms do (see Theorem~\ref{theo:learning-automata} in the next section).

The proof of Theorem~\ref{theo:learning-intervals} is technical, and we present the details in Sections~\ref{sec:properties} and~\ref{sec:algorithm}.
In the next section, we show how to apply this result in the context of symbolic automata.

    \section{Learning symbolic automata}\label{sec:learning-symbolic-automata}

Symbolic finite state automata (SFA) are an extension of classical finite state automata in which the transitions are replaced by predicates instead of the symbols from a finite alphabet. They were first introduced in \cite{sfa-van-noord} with natural language processing motivations and have been studied in recent decades because they enable applications over real-life scenarios that usually consider large or even infinite alphabets \cite{applications-symbolic-automata-Veanes, 10.1007/978-3-319-96145-3_23,learning-symbolic-automata}.

This paper addresses the learning of SFAs with one-dimensional predicates on $\bbQ$, built out of inequalities. In this section, we formally define SFAs with inequality formulas. Then, we present some applications of SFAs in practice. After this, we show how to apply Theorem~\ref{theo:learning-intervals} to MAT learning  of SFAs and discuss how it is compared with similar~works.

\paragraph{Inequality formulas.}
Let $x$ be a fixed variable. An \emph{inequality formula} has the following syntax:
\[
\varphi \ :=\ x < q \, \mid \, x > q \, \mid \, x \leq q \, \mid \, x \geq q \ \mid \, \varphi \vee \varphi \, \mid \, \varphi \wedge \varphi
\]
with $q \in \bbQ$. We say that a value $p \in \bbQ$ {\em satisfies} $\varphi$ if one can replace every appearance of $x$ in $\varphi$ by $p$ and obtain a binary formula that evaluates to true. We denote this by $p\models\varphi$. 
An inequality formula $\varphi$ naturally defines a function $\sem{\varphi}: \bbQ \mapsto \{0,1\}$ (also called a \emph{predicate}) such that $\sem{\varphi}(p) = 1$ if, and only if, $p \models \varphi$.
Two inequality formulas $\varphi$ and $\psi$ are {\em equivalent} if $\sem{\varphi} = \sem{\psi}$.  
As an example, consider:
\[
\varphi = 
\left(\tfrac{-2}{3} \leq x \,\wedge\, x \leq \tfrac{1}{3}\right) \ \vee \  \left(\tfrac{3}{2} < x\right)
\]
which defines $\sem{\varphi} = [\frac{-2}{3}, \frac{1}{2}] \cup (\frac{3}{2}, \infty)$ (similar to $\gamma_1$).

We write $\true$ or $\false$ as a shorthand for two fix formulas that are always $1$ or $0$, respectively. Further, for $\sim_1, \sim_2 \ \in \{<, \leq\}$, we write $q_1 \!\sim_1 \!x\! \sim_2\! q_2$ as a shorthand for $q_1 \sim_1 x \wedge x \sim_2 q_2$.
One can easily check that for every interval $I$, there exists a formula $\varphi_I$ that defines it (e.g., if $I = [\frac{-2}{3},\frac{1}{3}]$ then $\varphi_I := \frac{-2}{3} \leq x \leq \frac{1}{3}$). 
We call the rational numbers that appear in $\varphi$ the {\em coefficients of}~$\varphi$.
We define the \emph{size} of $\varphi$, denoted by $\size{\varphi}$, as the sum of the sizes of the coefficients and symbols appearing in $\varphi$.
Finally, we denote the set of all inequality formulas as $\IneqF$.

\paragraph{Symbolic Automata.} We introduce a subset of Symbolic Finite state Automata (SFA) that carry predicates that are represented by inequality formulas. The difference between SFA and classical finite automata is that finite automata have transitions of the form $s_1 \xrightarrow{a} s_2$ with $s_1,s_2$ states and $a$ a symbol in some finite alphabet; whereas SFA have transitions of the form $s_1 \xrightarrow{\varphi} s_2$ with $s_1,s_2$ states and $\varphi$ an inequality formula. Consequently, there is (possibly) an infinite number of values that can satisfy $\varphi$. Hence these type of transitions enable finite state automata to handle an infinite alphabet such as the set of rational numbers. 

\cristian{NOTAR QUE ahora estoy usando $S$ para el conjunto de estados y $s$, $s_1$ y $s_2$, y que $Q$ se parece a los racionales, y usamos $p$, $q$ para racionales.}

Formally, a (non-deterministic) \emph{Symbolic Finite state Automata with inequality predicates} (SFA) is a tuple:
\[
\cA = (S, \Delta, s_{0}, F) \tag{$*$}
\]
where $S$ is a finite set of states, $\Delta \subseteq S \times \IneqF \times S$ is a finite relation, $s_{0} \in S$ is the initial state, and $F \subseteq S$ is the set of final states. 
We say that $\Delta$ is the \emph{transition relation} of $\cA$ and $(s_1, \varphi, s_2) \in \Delta$ is a \emph{transition} of $\cA$ that we usually write as $s_{1} \xrightarrow{\varphi} s_{2}$.
A \emph{run} of $\cA$ over a string $w= q_{1} q_{2} \dots q_{n} \in \bbQ^{*}$ is a sequence of transitions:
\[
\rho \ := \ s_{0} \xrightarrow{\varphi_{1}} s_{1} \xrightarrow{\varphi_{2}} \dots \xrightarrow{\varphi_{n}} s_{n}
\]
such that $(s_{i-1}, \varphi_{i}, s_{i}) \in \Delta$ and $q_{i} \models \varphi_i$ for every $i \in \{1, \ldots, n\}$.
A \emph{run} is accepting if $s_{n} \in F$. The language accepted by $\cA$ is defined as $\cL(\cA) = \{w \in \bbQ^* \mid \cA \text{ accepts } w \}$.
We define $\size{\cA} = |S|+|\Delta| + \sum_{(s_1, \varphi, s_2) \in \Delta} \size{\varphi}$ as the \emph{size} of $\cA$.

An SFA $\cA$ of the form $(*)$ is \emph{deterministic} if, for all pairs of transitions $(s, \varphi, s_1),(s,\psi,s_2) \in \Delta$, if $\sem{\varphi} \cap \sem{\psi} \neq \emptyset$, then $\varphi = \psi$ and $s_1 = s_2$. 
In~\cite{minimization-symbolic-automata}, it was shown that for every SFA there exists an equivalent deterministic SFA of at most exponential size. As it is standard in the MAT learning setting of finite automata~\cite{angluin-learning-L*,learning-symbolic-automata}, in the sequel we assume that all SFAs are deterministic. 

In the following, we present two examples of how to use SFAs in practice that motivate their learning problem.

\paragraph{RGX.} In practice, regular expressions (RGX) process text documents where symbols are encoded in \emph{Unicode}, where each symbol is an integer in the range from $0$ to $1.114.111$ written in hexadecimal as \texttt{U+0000} to \texttt{U+10FFFF}. For instance, the digits $\{0, 1, \ldots, 9\}$ run from $48$ (\texttt{U+0030}) to $57$ (\texttt{U+0039}), the latin alphabet in uppercase $\{A, \ldots, Z\}$ runs from $65$ (\texttt{U+0041}) to $90$ (\texttt{U+005A}), and the latin alphabet in lowercase $\{a, \ldots, z\}$ runs from $97$ (\texttt{U+0061}) to $122$ (\texttt{U+007A}). Therefore, a text document in practice can be seen as a sequence of integers $q_1 \ldots q_n \in \bbN^*$ (note that although SFA and our results used $\bbQ$, they also applied to $\bbN$ given that $\bbN \subseteq \bbQ$).
For a simple example, suppose that we want to check a lowercase letter directly followed by an uppercase letter in the Latin alphabet. This pattern can be defined by a RGX\footnote{Notice that, in practice, RGX are unanchored, namely, they start the evaluation in any part of the document.} (in PCRE or Perl syntax) as:
\[
\texttt{[A-Z]\![a-z]}
\]
where \texttt{[A-Z]} and \texttt{[a-z]} are called \emph{char classes} denoting any symbol from $A$ to $Z$ (i.e. interval $[65,90]$) and from $a$ to $z$ (i.e. interval $[97,122]$). In Figure~\ref{fig:example-sfa-temperature} (1),
we show how to represent this RGX with an SFA by using inequalities. Note that this pattern can be defined with standard finite state automata; however, we will need a large number of transitions between states to encode the inequality formulas.%

\paragraph{Time series.} Consider a scenario where a sensor measures the temperature of the environment. More precisely, consider that the sensor emits a signal in the form of a sequence $q_1 q_2 \ldots q_n \in \bbQ^*$ where each $q_i \in \bbQ$ is a temperature measured in degrees Celsius and where the order represents the time (i.e., $q_1$ was measured before $q_2$, and so on). %

Consider that a user is interested in detecting the following pattern of the sensor signals: 
\[
\tfrac{13}{2} < t_{1} \leq \tfrac{23}{3} \ ; \  t_{2} > \tfrac{13}{1}
\]
which means that it wants to see a temperature $t_1$ in the range $(\frac{13}{2}, \frac{23}{3}]$ followed by a temperature $t_2$ over $13$.
We use~$;$ to denote that two events occur sequentially, but not necessarily consecutively. One can check that the SFA illustrated in Figure \ref{fig:example-sfa-temperature} (2) detects the previously defined pattern.%

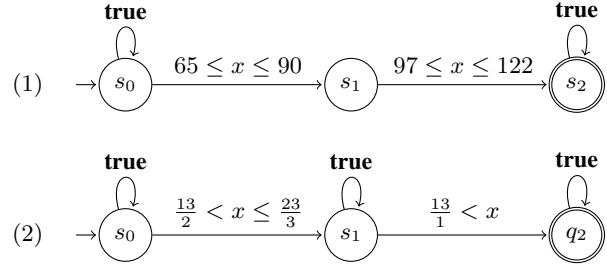
\begin{figure}[t] 
	\centering \small
	\begin{tikzpicture}[shorten >=1pt,node distance=3cm,on grid,auto,
		initial text= {},
		initial distance= {3mm}
		]
		\tikzstyle{every state}=[minimum size=0.7cm]

		\node[state,initial] (q0) {$s_{0}$};
		
		\node [node distance=1.3cm, left=of q0] {$(1)$};

		\node[state] (q1) [right=of q0] {$s_{1}$};
		
		\node[state, accepting]  (q2) [right=of q1]	{$s_{2}$};

	\path[->]
		(q0) edge node{$65 \leq x \leq 90$} (q1)
		(q1) edge node{$97 \leq x \leq 122$} (q2)
	;
	\draw[->]
		(q0) edge[loop above] 	node{$\true$} (q0)
		(q2) edge[loop above] 	node{$\true$} (q2)
	;
	
	\begin{scope}[yshift=-2cm]
		\node[state,initial] (q0) {$s_{0}$};
		
		\node [node distance=1.3cm, left=of q0] {$(2)$};
		
		\node[state] (q1) [right=of q0] {$s_{1}$};
		
		\node[state, accepting]  (q2) [right=of q1]	{$q_{2}$};
		
		\path[->]
		(q0) edge node{$\frac{13}{2} < x \leq \frac{23}{3}$} (q1)
		(q1) edge node{$\frac{13}{1} < x$} (q2);
		\draw[->]
		(q0) edge[loop above] 	node{$\true$} (q0)
		(q1) edge[loop above] 	node{$\true$} (q1)
		(q2) edge[loop above] 	node{$\true$} (q2)
		;
	\end{scope}
	\end{tikzpicture}
	\caption{Two SFA with inequality formulas. (1) is the SFA of a RGX and (2) is the SFA of a time series pattern.}
	\label{fig:example-sfa-temperature}
\end{figure}

The two previous examples are simplified cases to showcase the use of SFA with inequality formulas in practice. This setting can easily be extended to more complex scenarios, for example, where data symbols can be a tuple of the form $(a, q)$ where $a$ is a label and $q$ is a rational number. %
For the sake of presentation, we prefer to present SFA and their examples only over rational numbers, but one can easily extend them to other scenarios. 

\paragraph{Learning SFAs using MAT.}
The following result\footnote{\cite{10.1007/978-3-319-96145-3_23} presents Theorem~\ref{theo:learning-automata} with more details regarding the exact complexity of its MAT learning algorithm. Here, we oversimplified the upper bound on the size of the representation, which better fits our purposes.} presented in \cite{10.1007/978-3-319-96145-3_23} is crucial for deriving our result regarding MAT learning SFAs. 
\begin{theorem}[Theorem 1 in \cite{10.1007/978-3-319-96145-3_23}]\label{theo:learning-automata}
	Let $\galgo$ be an algorithm that learns $\IneqF$ using MAT with $f(n)$ membership queries and $g(n)$ equivalence queries. Then SFA can be learned using MAT with $n^2 \cdot f(n)+n^2 \cdot g(n) \cdot \log(m)$ membership queries and $n^2 \cdot g(n)$ equivalence queries where $m$ is the length of the longest counterexample.
\end{theorem}
In other words, a good MAT learning algorithm for the class of inequality formulas $\IneqF$ implies a good MAT learning algorithm for SFA with inequality predicates. One can easily note that inequality formulas are a special case of finite piecewise functions. Indeed, we already noticed that a formula $\varphi \in \IneqF$ defines a function $\sem{\varphi}: \bbQ \mapsto \Sigma$ where $\Sigma = \{0,1\}$. Furthermore, $\sem{\varphi}$ is a finite piecewise function and, if $\sem{\varphi}$ has a canonical representation of the form $(\dagger)$:
\[
\varphi^* := \bigvee_{j: \sem{\varphi}(I_j) = 1} \varphi_{I_j}
\]
where $\varphi_{I}$ is the formula representing the interval $I$. Then, $\size{\varphi^*} \in \bigo(\size{\sem{\varphi}})$, that is, the minimum size of a representation in $\IneqF$ for $\varphi$ is proportional to the size of the canonical representation of $\sem{\varphi}$. Then, by applying Theorem~\ref{theo:learning-intervals} for learning finite piecewise functions with Theorem~\ref{theo:learning-automata} for learning SFAs we get the following corollary.

\begin{corollary}\label{cor:fully-learning-SFA}
	SFAs can be learned using MAT with $n^3 \cdot \log(m)$ membership queries and $n^3$ equivalence queries where $m$ is the length of the longest counterexample.
\end{corollary}

It is important to note that previous works like \cite{learning-symbolic-automata, 10.1007/978-3-319-96145-3_23,Mens_2015} settled the basis on MAT learning SFAs (e.g., Theorem~\ref{theo:learning-automata}); however, their results rely on simple predicates (e.g., equality predicates) or strong assumptions over the teacher (e.g., the equivalence query always returns the minimal counterexample). Instead, to the best of our knowledge, Corollary~\ref{cor:fully-learning-SFA} is the first result that shows \emph{fully} MAT learning of SFA over non-trivial predicates. Furthermore, it is efficient in the sense that Theorem~\ref{theo:learning-intervals} is the best that one could achieve for combining it with Theorem~\ref{theo:learning-automata}.

	\section{The Stern-Brocot tree and break links}\label{sec:properties}

In this section, we present the Stern-Brocot tree of rational numbers and some of its properties. Then we introduce the concept of a break link and its connection with convergent fractions, which will be fundamental in supporting the efficient learning algorithm of finite piecewise functions.

\paragraph{The Stern-Brocot tree.} The Stern-Brocot tree ($\SBtree$) is a construction to represent the set of positive rational numbers presented by \cite{Stern1858,Brocot1861}. In this paper we consider an extended version of this tree that contains the whole set of rational numbers $\bbQ$. %
The theory for the Stern-Brocot tree is well-studied, so we present here the main definitions that are important for our algorithmic result (see \cite{graham94} for further details).

The \emph{Stern-Brocot tree} is a directed rooted tree $\SBtree = (V(\SBtree), E(\SBtree))$ such that $\SBtree$ is an infinite complete binary tree where $V(\SBtree) = \bbQ$ are the set of vertices (written as irreducible fractions), $E(\SBtree) = \{(u, \leftChild{u}), (u, \rightChild{u}) \mid u \in V(\SBtree)\}$ is the set of edges, and $\frac{0}{1}$ is the root.
Figure \ref{fig:stern-brocot-tree} illustrates the structure of the Stern-Brocot tree up to depth 4.
Every vertex $u$ in $V(\SBtree)$ has exactly two children, the left child $\leftChild{u}$ and the right child $\rightChild{u}$, which are placed from left to right, following the left-to-right order of their children.

The construction of $\SBtree$ is as follows: for each node $u = \frac{a}{b}$ we define the {\em boundary} relation $B \subseteq \bbQinf \times \bbQ \times \bbQinf$ as the smallest set that satisfies that $(\frac{-1}{0},\frac{0}{1},\frac{1}{0}) \in B$, and that if $(\frac{c_1}{d_1},\frac{a}{b},\frac{c_2}{d_2}) \in B$, then $(\frac{c_1}{d_1},\frac{c_1+a}{d_1+b},\frac{a}{b}) \in B$ and $(\frac{a}{b},\frac{a+c_2}{b+d_2}, \frac{c_2}{d_2}) \in B$. 
It can be shown that every irreducible fraction, except $\frac{-1}{0}$ and $\frac{1}{0}$, appears as the middle component of some triple in $B$ exactly once.
With the relation $B$, we define the functions $\operatorname{left},\operatorname{right}: \bbQ\to \bbQ$ of $\SBtree$ by: 
\[
\leftChild{\frac{a}{b}} = \frac{a+c_1}{b+d_1} \ \ \text{ and } \ \ \rightChild{\frac{a}{b}} = \frac{a+c_2}{b+d_2},
\] 
for every $\frac{a}{b} \in \bbQ$ such that $(\frac{c_1}{d_1},\frac{a}{b},\frac{c_2}{d_2})\in B$.
One can check that this relation holds in Figure \ref{fig:stern-brocot-tree}.
\cristian{TODO: Agrega esta definición de parent al apendice cuando se use por primera vez (o en esa demostración): "Further, for every $q \in \bbQ \setminus \{\frac{0}{1}\}$ we define $\parent{q}$ as the only rational $p \in \bbQ$ such that $\leftChild{p} = q$ or $\rightChild{p} = q$."}

\begin{figure}[t]
  \centering
  \scalebox{0.9}{
  \begin{tikzpicture}  \tikzset{
  node line/.style={thick}
}

\forestset{
  nodeGreen/.style={fill=gray!60!},
  nodeRed/.style={}
}

  \node[inner sep=0pt] {
    \begin{forest}
      for tree={
        rectangle,
        minimum size=1.5em,
        inner sep=1pt,
        l=1mm, s sep=0.2mm,
        anchor=center,
        math content
      }
      [
        {\tfrac{0}{1}}, nodeGreen
        [
          {\tfrac{-1}{1}}, nodeRed
          [
            {\tfrac{-2}{1}}, nodeRed
            [
              {\tfrac{-3}{1}}, nodeRed
              [
                {\tfrac{-4}{1}}, nodeRed
              ]
              [
                {\tfrac{-5}{2}}, nodeRed
              ]
            ]
            [
              {\tfrac{-3}{2}}, nodeRed
              [
                {\tfrac{-5}{3}}, nodeRed
              ]
              [
                {\tfrac{-4}{3}}, nodeRed
              ]
            ]
          ]
          [
            {\tfrac{-1}{2}}, nodeGreen
            [
              {\tfrac{-2}{3}}, nodeGreen
              [
                {\tfrac{-3}{4}}, nodeRed
              ]
              [
                {\tfrac{-3}{5}}, nodeGreen
              ]
            ]
            [
              {\tfrac{-1}{3}}, nodeGreen
              [
                {\tfrac{-2}{5}}, nodeGreen
              ]
              [
                {\tfrac{-1}{4}}, nodeGreen
              ]
            ]
          ]
        ]
        [
          {\tfrac{1}{1}}, nodeRed
          [
            {\tfrac{1}{2}}, nodeGreen
            [
              {\tfrac{1}{3}}, nodeGreen
              [
                {\tfrac{1}{4}}, nodeGreen
              ]
              [
                {\tfrac{2}{5}}, nodeGreen
              ]
            ]
            [
              {\tfrac{2}{3}}, nodeRed
              [
                {\tfrac{3}{5}}, nodeRed
              ]
              [
                {\tfrac{3}{4}}, nodeRed
              ]
            ]
          ]
          [
            {\tfrac{2}{1}}, nodeGreen
            [
              {\tfrac{3}{2}}, nodeRed
              [
                {\tfrac{4}{3}}, nodeRed
              ]
              [
                {\tfrac{5}{3}}, nodeGreen
              ]
            ]
            [
              {\tfrac{3}{1}}, nodeGreen
              [
                {\tfrac{5}{2}}, nodeGreen
              ]
              [
                {\tfrac{4}{1}}, nodeGreen
              ]
            ]
          ]
        ]
      ]
    \end{forest}
  };
\end{tikzpicture}
}

  \centering
\[  \rep_1 \,=\, \left((\tfrac{-1}{0}, \tfrac{-2}{3}), B\right)\left([\tfrac{-2}{3}, \tfrac{1}{2}], A\right)\left((\tfrac{1}{2}, \tfrac{3}{2}], B\right)\left((\tfrac{3}{2}, \tfrac{1}{0}), A\right)\]

  \caption{Top: The finite piecewise function $\gamma_1$ (see Section~\ref{sec:setting}) shown in a Stern–Brocot tree up to depth 4. Nodes $q$ for which $\gamma(q)=A$ (resp. $B$) are represented by gray (resp. white)  blocks. Bottom: The canonical representation of $\gamma_1$.}
  
  \label{fig:stern-brocot-tree}
\end{figure}

Since $\SBtree$ forms a binary tree and 
every $\frac{a}{b}\in \bbQ$ is uniquely represented as a vertex in $V(\SBtree)$, then there is a unique path from the root $\frac{0}{1}$ to $\frac{a}{b}$. We can encode this path as a string in $\{L, R\}^*$ of the left-child ($L$) and right-child ($R$) moves from $\frac{0}{1}$ to $\frac{a}{b}$. 
For instance, one can check in Figure~\ref{fig:stern-brocot-tree} that $LRL$ represents $\frac{-2}{3}$ and $RLL$ represents $\frac{1}{3}$. In particular, negative numbers starts with $L$, positive numbers with $R$, and $\frac{0}{1}$ is represented by the empty string. Furthermore, we can succinctly encode a string in $\{L, R\}^*$ by its \emph{run-length encoding} of the contiguous sequences of $R$ an $L$ symbols.  For example, the string
$RLLLLRLLRRRL$ (which represents $\frac{23}{108}$) can be grouped as
$R^1L^4R^1L^2R^3L^1\!$.
In general, for every $q \in \bbQ$ we can define its \emph{SB-encoding}:
\[
\SBenc{q} \ = \ \  \sim\!\![a_1, \ldots, a_n]
\]
where $\sim \, \in \{+,-\}$, $n \geq 1$, and $a_i \geq 1$ such that:
\begin{itemize}
	\item if $\sim \,= +$ and $n$ is even, then $S = R^{a_1} L^{a_2} \ldots R^{a_{n-1}} L^{a_{n}}$;
	\item if $\sim \,= +$ and $n$ is odd, then $S = R^{a_1} L^{a_2} \ldots L^{a_{n-1}} R^{a_{n}}$;
	\item if $\sim \,= -$ and $n$ is even, then $S = L^{a_1} R^{a_2} \ldots L^{a_{n-1}} R^{a_{n}}$;
	\item if $\sim \,= -$ and $n$ is odd, then $S = L^{a_1} R^{a_2} \ldots R^{a_{n-1}} L^{a_{n}}$;
\end{itemize}
and $S$ is the path from $\frac{0}{1}$ to $q$ in $\SBtree$. In particular, we define the empty array as the SB-encoding of $\frac{0}{1}$. For example, from Figure~\ref{fig:stern-brocot-tree} one can check that $\SBenc{\frac{-2}{3}} = -[1,1,1]$ and $\SBenc{\frac{1}{3}} = +[1,2]$.
It is worth noticing that the SB-encoding is directly related to the \emph{continuous fraction} representation of rational numbers (see~\cite{graham94})---yet we will present the notion solely as the path of left- and right-child moves in $\SBtree$, as this suits our purposes better. 

An important property of the SB-encoding of rational numbers is that its size its not much bigger than the irreducible representation of the number.
\begin{proposition}\label{prop:continued-fraction-representation}
	Given any rational number $q=\frac{a}{b}$ and its corresponding SB-encoding $\sim\!\![a_{1}, \dots, a_{n}]$, 
	it holds that $\sum_{i=1}^{n}\log(a_{i})\in\bigo(\size{q})$.
\end{proposition}

Notice that every array $\sim\!\![a_1, \ldots, a_n]$ uniquely represents a rational number in $\SBtree$. 
For every $q \in \bbQ$ with $\SBenc{q} = \,\sim\!\![a_1, \ldots, a_n]$ 
we call the rational number $p_m$ with $\SBenc{p_m} =\, \sim\!\![a_{1}, \dots, a_{m}]$ (for $1 \leq m \leq n$) the \emph{$m$-th convergent fraction} of $q$.
In addition, we say that $p_m$ is a strict convergent fraction of $q$ if $m < n$. 
For instance, $\frac{-1}{1}$ and $\frac{-1}{2}$ are strict convergent fractions of $\frac{-2}{3}$ and $\frac{1}{1}$ is the only strict convergent fraction of $\frac{1}{3}$. 
One can think of the convergent fractions of a rational number $q$ as the nodes at the turning points in the path to $q$ from the root in $\SBtree$. %
Convergent fractions will be key to our main characterization.

\rodrigo{creo que falta actualizar la notación de canonical representation de $R$ a $\cR$ en esta sección}
\cristian{Actualizada.}

\paragraph{Break links.} For learning an unknown finite piecewise function $\gamma$ with canonical representation $\rep$ of the form $(\dagger)$, 
the objective is to identify every interval $I_j$ and the corresponding value $\gamma(I_j)$.
In the next section, we present a learning algorithm that, using the Stern-Brocot tree, constructs $\rep$ by finding the convergent fractions of each endpoint of the intervals in~$\rep$. 
In this subsection, we introduce a notion inside the Stern-Brocot tree, called \emph{break links}, that we will use as a proxy to find all the values that appear in $\rep$.

Formally, given an interval $I = \langle p, q\rangle$, we denote by $\bounds(I) = \{p, q\}$ , namely, the lower and upper bounds of $I$. 
For a finite piecewise function $\gamma$ with canonical representation $\rep$ like ($\dagger$), we define $\bounds(\gamma) = \{\bounds(I_j) \mid j \leq k\} \setminus \{-\infty, \infty\}$. 
That is, $\bounds(\gamma)$ contains the bounding values inside $\gamma$ except $-\infty$ and $\infty$. 
Therefore, the main goal in our learning algorithm will be to find all the values in $\bounds(\gamma)$ for $\gamma$.   

In the following, we denote the \emph{unordered interval} of two rational numbers $p,q$ as: $\unorderInt{p,q} = [\min(p,q), \max(p,q)]$.
Let $\gamma$ be a finite piecewise function and $(p,q)$ be an edge in $\SBtree$.
We say that $(p,q)$ is a {\em break link}  of $\gamma$ if $p$ is $\frac{0}{1}$ or $\unorderInt{p,q}$ is not monochromatic with respect to $\gamma$, meaning that there exist $x,y \in \unorderInt{p,q}$, such that $\gamma(x) \neq \gamma(y)$. 
In Figure~\ref{fig:stern-brocot-tree}, we show the Stern-Brocot tree up to depth 4
along with the finite piecewise function $\gamma_1$ introduced in Section~\ref{sec:setting} where nodes $q$ for which $\gamma(q)=A$ (resp. $B$) are represented by gray (resp. white) blocks.
One can check that $(\tfrac{0}{1}, \tfrac{1}{1})$ and $(\tfrac{1}{2}, \tfrac{2}{3})$ are break links of $\gamma_1$ but $(\tfrac{-1}{1}, \tfrac{-2}{1})$ and   $(\tfrac{1}{2}, \tfrac{1}{3})$ are not.
\cristian{Saque la definición de "sees" y también quite el example environment. No estamos usando example environment en el paper, entonces por consistencia mejor sacarlo.}

Next, we present the relation between break links and the convergent fractions of the bounds in a piecewise function~$\gamma$.

\begin{proposition}\label{prop:break-links-convergent-equivalence}
	Given a finite piecewise function $\gamma$, (1) for every convergent fraction $q$ of a bound in $\bounds(\gamma)$, there exists a break link $(p_1,p_2)$ of $\gamma$ such that $q\in \{p_1,p_2\}$; and (2) for every break link $(p_1,p_2)$ of $\gamma$ at least one among $p_1$ and $p_2$ is a convergent fraction of some bound in $\bounds(\gamma)$.
\end{proposition}

The previous proposition is crucial for our learning algorithm to find the endpoints in $\gamma$. 
Given that every value in $\bbQ$ is a convergent fraction of itself,
(1) implies that, if we find all break links of $\gamma$ in $\SBtree$, then we will have all the elements in  $\bounds(\gamma)$. 
Furthermore, by (2) and  Propositions~\ref{prop:continued-fraction-representation} the number of break links of $\gamma$ will not be bigger than $\size{\gamma}$.  
In the next section, we will show how to use this connection to learn any finite piecewise function $\gamma$.

    \section{The learning algorithm}\label{sec:algorithm}

The goal of this section is to present the learning algorithm that, given a membership oracle $\membOracleName$ and an equivalence oracle $\equivOracleName$ of an unknown finite piecewise function $\gamma$, computes its representation $\rep$ of the form ($\dagger$).
The algorithm works by incrementally adding break links of $\gamma$ until every break link has been discovered by the algorithm. Then, from the break links we construct the final hypothesis. 
We will start the section by describing the data structure and an algorithm to build a representation from the collected data. Then, we present the learning algorithm and prove its correctness. 

\paragraph{Building a consistent representation.}
As already mentioned, the main strategy of the algorithm is to find all break links $(p,q)$ of $\gamma$.
Instead of keeping pairs $(p,q)$, we will maintain the points $p$ and $q$ in a list $D$ of the form:
\[
D = (q_1, A_1), (q_2, A_2), \ldots, (q_m, A_m) \tag{$\ddagger$}
\] 
where $q_i \in \bbQ$ and $\gamma(q_i) = A_i$ for every $i \leq m$, and $q_1 < q_2 < \ldots < q_m$. Furthermore, every $q_i$ is part of a break link of $\gamma$, namely, $q_i \in \{p,q\}$ for some break link $(p,q)$ of~$\gamma$. We call $D$ the \emph{data structure} of our learning algorithm. 

Our first task for the learning algorithm is to show how to construct a 
representation $\rep$ from a given data structure $D$. 
Specifically, we say that a representation $\rep$ is \emph{consistent} with a data structure $D$ like ($\ddagger$) if $\rep(q_i) = A_i$ for every $i \leq m$. 
For a finite piecewise function $\gamma$ with canonical representation $\rep$ like ($\dagger$), 
we say that $q \in \bounds(\gamma)$ is a \emph{left-bound} (\emph{right-bound}) of $\gamma$ if $q \in I_j$ and $q = \inf(I_j)$ (resp. $q = \sup(I_j)$) for some $j \leq k$. In other words, $I_j$ is of the form $[q, p\rangle$ (resp. $\langle p, q]$) for some $p \in \bbQinf$. Note that every $q \in \bounds(\gamma)$ is either a left- or a right-bound of $\gamma$ or both (e.g., $I_j = [q,q]$). 

So, how do we build a consistent representation $\rep$ from $D$? 
We use a strategy to determine which of the values in $D$ are endpoints in $\bounds(\gamma)$ and, additionally, which are left- or right-bounds (or both).
Intuitively, if $q_{i-1}$ is a descendant of $q_{i}$ in $\SBtree$ and $\gamma(q_{i-1}) \neq \gamma(q_{i})$ for some $i > 1$, then this is evidence that $q_{i} \in \bounds(\gamma)$ and, furthermore, $q_{i}$ is a left-bound of $\gamma$. Similarly, if $q_{i}$ is an ancestor of $q_{i+1}$ and $\gamma(q_{i}) \neq \gamma(q_{i+1})$ for some $i < m$, then $q_{i}$ is a right-bound of $\gamma$. This intuition, in fact, is guaranteed to completely characterize $\bounds(\gamma)$ when $D$ contains all break links of $\gamma$.
\begin{proposition}\label{prop:hypothesis-construction-correctness}
	Assume that $D$ like ($\ddagger$) contains all break links of $\gamma$. Then (1) $q_i$ is a left-bound  of $\gamma$ iff $q_{i-1}$ is a descendant of $q_i$ and $A_{i-1} \neq A_{i}$, and (2) $q_i$ is a right-bound  of $\gamma$ iff $q_{i}$ is an ancestor of $q_{i+1}$ and $A_{i} \neq A_{i+1}$.
\end{proposition}
\begin{algorithm}[tb]
	\caption{Construct a representation $\rep$ from $D$}
	\label{alg:hip-constructions}
	\begin{algorithmic}[1]
		\REQUIRE A data structure $D$ like $(\ddagger)$.
		\ENSURE A representation $\rep$ consistent with $D$.  \vspace{1mm}
		
		\PROCEDURE{$\constructRepresentationName$}{$D$}
		\STATE $\rep \gets \emptyset$ 
		\STATE $\boldsymbol{\langle} \gets ($\,, \ ${\sf leftBnd} \gets -\infty$, \ ${\sf currLabel} \gets A_1$ \label{in-alg:minusinftycase}
		\FOR {$i \in \{1, \ldots, m\}$ }
		\IF {$q_{i-1}$ is descendant of $q_i$ ${\sf and}$ $A_{i-1} \neq A_i$}  \label{in-alg:firstIf}
		\STATE $\rep \gets \rep \cup ( \, \boldsymbol{\langle}\,  {\sf leftBnd}, q_i), {\sf currLabel})$
		\STATE $\boldsymbol{\langle} \gets [$\,, \  ${\sf leftBnd} \gets q_i$, \ ${\sf currLabel} \gets A_i$ \label{in-alg:firstIfEnd}
		\ENDIF
		\IF {$q_{i}$ is an ancestor of $q_{i+1}$ ${\sf and}$ $A_i \neq A_{i+1}$} \label{in-alg:secondIf}
		\STATE $\rep \gets \rep \cup ( \, \boldsymbol{\langle}\,  {\sf leftBnd}, q_i], {\sf b_i})$
		\STATE $\boldsymbol{\langle} \gets ($\,, \ ${\sf leftBnd} \gets q_i$, \ ${\sf currLabel} \gets A_{i+1}$ \label{in-alg:secondIfEnd}
		\ENDIF
		\ENDFOR
		\STATE $\rep \gets \rep \cup ( \, \boldsymbol{\langle}\, {\sf leftBnd}, \infty), A_{m})$ \label{in-alg:inftycase}
		\STATE \textbf{return} $\rep$
		\ENDPROCEDURE
	\end{algorithmic}
\end{algorithm}

In Algorithm~\ref{alg:hip-constructions}, we implement this strategy for constructing a representation from $D$ where lines \ref{in-alg:firstIf}-\ref{in-alg:firstIfEnd} encode (1) and lines \ref{in-alg:secondIf}-\ref{in-alg:secondIfEnd} encode (2) plus the border cases (i.e., $-\infty$ and $\infty$) which we resolve in lines~\ref{in-alg:minusinftycase} and \ref{in-alg:inftycase}. Notice that before $D$ has all break links, $\constructRepresentation{D}$ will not necessarily produce the correct representation for $\gamma$. Nevertheless, as the following result shows, it always produces a consistent representation of $D$. 
\begin{theorem}\label{theo:hip-construction}
	Algorithm~\ref{alg:hip-constructions} outputs a representation $\rep$ that is consistent with $D$. Moreover, if $D$ contains all break links of~$\gamma$, then $\rep = \gamma$. 
\end{theorem}
By the previous result, we know that the representation $\rep$ is always consistent with $D$ and, moreover, if we aim to find all break links, then eventually we will have that $\rep = \gamma$. 

\paragraph{Finding all break links.}
The learning algorithm is presented in Algorithm~\ref{alg:learning-algo}. It receives as input a membership oracle $\membOracleName$ and an equivalence oracle $\equivOracleName$, both compatible with an unknown finite piecewise function $\gamma$ in $\cC$. Its goal is to construct a representation $\rep$ for $\gamma$ of the form~($\dagger$) after a finite number of iterations. 
This is done by maintaining a data structure $D$ that stores all the break links of $\gamma$ that have been found up to the current iteration.
Since $\gamma$ is a finite piecewise function, 
the interval partition $I_1, I_2, \dots, I_k$ of $\bbQ$ must be finite and, furthermore, every endpoint $q$ in $\bounds(\gamma)$ has a finite encoding $\SBenc{q}$. 
Therefore, by Proposition~\ref{prop:break-links-convergent-equivalence} and Theorem~\ref{theo:hip-construction}, the number of break links of $\gamma$ must be finite, and, after collecting all of them in $D$, we can construct a representation $\rep$ of $\gamma$.

\begin{algorithm}[tb]
	\caption{The Learning Algorithm}\label{alg:learning-algo}
	\begin{algorithmic}[1]
		\REQUIRE A Membership Oracle $\membOracleName$ and an Equivalence Oracle $\equivOracleName$ compatible with some finite p.w. function $\gamma$.
		\ENSURE A representation $\rep$ equivalent to $\gamma$. \vspace{1mm}
		\STATE $D \gets \{(\frac{0}{1}, \membOracle{\frac{0}{1}})\}$ \label{in-alg:start-set}
		
		\WHILE{\textbf{true}} \label{in-alg:start-loop}
		
		\STATE $\rep \gets \constructRepresentation{D}$ \label{in-alg:construct-representation}
		\STATE $({\sf ans}, q^*) \gets \equivOracleName(\rep)$ \label{in-alg:ask-equivalence}
		
		\IF {${\sf ans} = \textbf{true}$} \label{in-alg:check-equivalence}
		\STATE \textbf{return} $\rep$ \label{in-alg:return-hypothesis}
		\ENDIF \label{in-alg:end-if}
		
		\STATE $r \gets \findClosestAncestor{q^*\!}{D}$ \label{in-alg:find-closest-ancestor}
		\STATE $(p, q) \gets \findBreakLink{r}{q^*\!}{D}$ \label{in-alg:find-break-link}
		\STATE $D \gets D\cup \{(p, \membOracle{p}), (q, \membOracle{q})\}$ \label{in-alg:add-new-break-link}
		
		\ENDWHILE \label{in-alg:end-loop}
	\end{algorithmic}
\end{algorithm}

Algorithm~\ref{alg:learning-algo} works by first adding a pair $(0, \gamma(0))$ into $D$ (line~\ref{in-alg:start-set}), and then, at each iteration, 
adding the components $p$ and $q$ of a new break link $(p,q)$ of $\gamma$ into $D$, along with the labels $\gamma(p)$ and $\gamma(q)$ (line \ref{in-alg:add-new-break-link}). 
Each iteration does the following:
First, it constructs a representation $\rep$ using Algorithm~\ref{alg:hip-constructions} and the array $D$ (line~\ref{in-alg:construct-representation}). Then, it asks the equivalence oracle $\equivOracleName$ whether $\rep = \gamma$; if it is true, the algorithm ends
(lines~\ref{in-alg:ask-equivalence}--\ref{in-alg:end-if}). 
If not, this means that we are missing some break link in $D$ and, then, in lines~\ref{in-alg:find-closest-ancestor}--\ref{in-alg:find-break-link} the algorithm searches the new break link for $\gamma$ in $\SBtree$. Finally, in line~\ref{in-alg:add-new-break-link}, it adds the information of the new break link into $D$, maintaining the array ordered, and avoiding adding any existing tuple~again.
By Theorem~\ref{theo:hip-construction}, we know that if the equivalence oracle returns a counterexample $q^{*}$, then $D$ is missing some break link of $\gamma$. Thus, if we prove that Algorithm~\ref{alg:learning-algo} always adds a new break link in lines~\ref{in-alg:find-closest-ancestor}--\ref{in-alg:add-new-break-link}, then its correctness will follow. 

The question that remains to answer is where to find a new break link given the counterexample $q^*$. Towards this goal, Algorithm~\ref{alg:learning-algo} first computes the closest ancestor of $q^*$ in $D$ (line \ref{in-alg:find-closest-ancestor}), namely, the element $r\in\{q_1, \ldots, q_m\}$ that is the closest ancestor to $q^{*}$ in $\SBtree$ assuming that $D$ has form $(\ddagger)$. One can show that one can compute method $\findClosestAncestor{q^*\!}{D}$ in time $\bigo(|D|)$ and without calling $\membOracleName$ or $\equivOracleName$.
To this end, we show that, if $D$ is missing a break link, then one can be found in the left branch or right branch that stems from $r$, depending on whether $q^* < r$ or $q^* > r$, respectively.
The case $q^* = r$ would never happen since $D$ is assumed to be consistent with $\gamma$.

\begin{proposition}\label{prop:find-breaklink}
	Assume that $\rep \neq \gamma$ in line~\ref{in-alg:check-equivalence} of Algorithm~\ref{alg:learning-algo} and $r$ is the closest ancestor of $q^*$ in $D$. Then there is a missing break link in the left branch of $r$ when $q^* < r$, or in the right branch when $q^* > r$.
\end{proposition}

\begin{algorithm}[tb]
	\caption{Find a break link of $\gamma$ from a node in $\SBtree$}
	\label{alg:find-break-link}
	\begin{algorithmic}[1] 
		\REQUIRE An origin node $r$, counterexample $q^{*}$ and array $D$.  
		\ENSURE A break link $(p, q)$ of $\gamma$ not present in $D$. \vspace{1mm} \PROCEDURE{$\findBreakLinkName$}{$r, q^*, D$}
		\IF {$q^*\!< r$}
		\STATE $(p, q) \gets \searchLeft{r}{q^{*}}{D}$
		\ELSE
		\STATE $(p, q) \gets \searchRight{r}{q^{*}}{D}$
		\ENDIF
		\STATE \textbf{return} $(p,q)$
		\ENDPROCEDURE
	\end{algorithmic}
\end{algorithm}

The method $\findBreakLink{r}{q^*\!}{D}$ in Algorithm~\ref{alg:find-break-link} exploits Proposition~\ref{prop:find-breaklink} to find a break link.
Specifically, it works as follows: given an origin node $r$, a counterexample $q^{*}$ and array $D$, it searches over the left branch or right branch that stem from $r$ using two subroutines $\operatorname{searchLeft}$ and $\operatorname{searchRight}$, respectively (see the Appendix for the implementation of both methods).
Each subroutine uses exponential search and membership queries to efficiently find a break link over the branches. If the distance from $r$ to an undiscovered break link of $\gamma$  is $k$, then the number of membership queries it takes to find this break link is in~$\bigo(\log(k))$.

\begin{theorem}\label{theo:find-breaklink}
	Assume that $\rep \neq \gamma$ in line~\ref{in-alg:check-equivalence} of Algorithm~\ref{alg:learning-algo} and $r$ is the closest ancestor of $q^*$ that appears in $D$. The method~$\findBreakLink{r}{q^*\!}{D}$ returns an undiscovered break link $(p,q)$ of $\gamma$ not present in $D$ making at most $\bigo(\log(k))$ membership queries to $\membOracleName$ where $k$ is the distance in $\SBtree$ between $r$ and $p$.
\end{theorem}

\martin{comprimi un poco lo de abajo}

Finally, let us analyze the number of membership and equivalence queries needed to learn an unknown finite piecewise function $\gamma$ using Algorithm~\ref{alg:learning-algo}.
The number of membership queries is the total of queries to MQ made in all calls to $\operatorname{FindBreakLink}$.
We see that each $q \in \bounds(\gamma)$ with SB-encoding $\sim\!\![a_{1}, \dots, a_{n}]$, by Proposition~\ref{prop:break-links-convergent-equivalence} and Theorem~\ref{theo:find-breaklink}, contributes $\bigo(\log(a_i))$ calls to find the break link at $a_i$.
Thus, the total number of membership queries is in $\bigo(\sum_{i=1}^{n}\log(a_i))$ for $q$ and,  by Proposition~\ref{prop:continued-fraction-representation}, we find that the number of membership queries is in $\bigo(\size{\gamma})$.

The number of equivalence queries used in Algorithm~\ref{alg:learning-algo} is bounded by the number of iterations, which is equal to the number of break links for $\gamma$. 
As we reasoned above, this number is in $\bigo(\size{\gamma})$ by Proposition~\ref{prop:continued-fraction-representation}.

\begin{theorem}\label{theo:correctness-learning-algo}
	Algorithm~\ref{alg:learning-algo} returns the canonical representation $\rep$ of a finite piecewise function $\gamma$ using no more than $\bigo(\size{\gamma})$ membership queries and no more than $\bigo(\size{\gamma})$ equivalence~queries.
\end{theorem}

	\section{Conclusions and future work}\label{sec:conclusions}

We presented a MAT learning algorithm for finite piecewise functions that directly applies to the learning of symbolic automata with interval formulas. Furthermore, this algorithm is efficient in the sense that it takes at most a linear number of queries to the Teacher concerning the size of the concept.

The techniques developed in this paper could have the potential to work in other problems regarding MAT learning. One open problem is to learn a finite piecewise function over $\bbQ^k$ (i.e., $k$-dimensions) where a finite number of square regions specify the function. This problem also has applications in symbolic automata, where, for example, time series are sequences of data tuples (e.g., the sensor measures temperature and humidity together). Further, the techniques proposed in this paper could have application in learning timed automata~\cite{learning-one-clock-timed-automata-an-jie}, where clock conditions can be seen as a special case of finite piecewise functions.

	\section*{Acknowledgements}
	
	The work of Muñoz and Riveros was supported by ANID – Millennium Science Initiative Program – Code ICN17\_002. Riveros was also supported by ANID Fondecyt Regular project 1230935.
	The work of Toro Icarte was supported by the National Center for Artificial Intelligence CENIA FB210017 (Basal ANID) and Fondecyt project 11230762.	
	
    \bibliographystyle{abbrv}
    \bibliography{ref.bib}

    \appendix

    \newpage
    \onecolumn

    \section{Proofs of Section 5}

\subsection{Proof of Proposition~\ref{prop:continued-fraction-representation}}

Towards the statement, we will prove a more precise bound:
\begin{proposition}
	Given any rational number $q=\frac{a}{b}$ and its corresponding SB-encoding $\sim\![a_{1}, \dots, a_{n}]$, it holds that: 
	\[\sum_{i=1}^{n}\log(a_{i}) \leq \size{|q|+1}.\]
\end{proposition}
We first point out that if $q = \frac{a}{b}$ with $a \geq 1$ and $b \geq 1$, then $\size{q+1} \leq 2\cdot\size{q} + 1$---and the case $q < 0$ is analogous---so the statement in the body is indeed implied by the formulation above.
\begin{proof}
	For this result, we will introduce the correspondence between the SB-encoding of a number and continued fractions.
	
	We state without proof that if $q$ is represented---as stated in Section~\ref{sec:properties}---by $R^{a_1}L^{a_2}\ldots D^{a_n}$, where $D\in\{L,R\}$ and each $a_i > 0$, then $q$ is equal to:
	\[
		q = a_1 + \frac{1}{a_2 + \frac{1}{a_3 + \cdots \frac{1}{a_{n-1}+\frac{1}{a_n+1}}}} - 1,
	\]
	and if $q$ is represented by $L^{a_1}R^{a_2}\ldots D^{a_n}$, where $D\in\{L,R\}$, then $q$ is equal to:
	\[
		q = -\left(a_1 + \frac{1}{a_2 + \frac{1}{a_3 + \cdots \frac{1}{a_{n-1}+\frac{1}{a_n+1}}}} - 1\right).
	\]
	For clarity, we note that the case where $n = 1$ is also considered and if $q$ is represented by $R^{a_1}$, then $q = (a_1 + 1) - 1 = a_1$, and if it is represented by $L^{a_1}$, then $q = -((a_1 + 1) - 1) = -a_1$.
	
	We address the case $q = 0$ independently, where the statement follows trivially.
	Now, assume without loss of generality that $q > 0$; thus, $q$ is represented by a string $R^{a_1}L^{a_2}\ldots D^{a_n}$.
	We will prove the statement by induction on $n$---i.e., the length of the SB-encoding of $q$.
	
	For our base case, we note that if $q$ is represented by $R^{a_1}$, then $q = a_1$, and $\log(a_1) \leq \size{q}$ follows by definition.
	
	For the inductive case, we assume that the statement holds for all rationals that can be represented with a string of size less than $n$. Let $q$ have an SB-encoding $+[a_{1}, \dots, a_{n}]$, and let $p$ be the rational with encoding $+[a_{2}, \dots, a_{n}]$. We note that
	\[
		p = a_2 + \frac{1}{a_3 + \cdots \frac{1}{a_{n-1}+\frac{1}{a_n+1}}} - 1,
	\]
	and so, $q+1 = a_1 + \frac{1}{p+1}$. Let $q+1 = \frac{c}{d}$ as an irreducible fraction. We note that since $a_1 \geq 1$ and $0 < \frac{1}{p} \leq 1$, it holds that $a_1 = \lfloor\frac{c}{d}\rfloor$. Using the equality $x = y\lfloor\frac{x}{y}\rfloor + x \!\!\mod y$, it follows that $p+1 = \frac{c\!\!\mod d}{d}$, which is also irreducible.
	
	Towards the inductive statement, let us note that $\log(a_1) = \log(\lfloor \frac{c}{d}\rfloor) \leq \log(\frac{c}{d}) = \log(c) - \log(d)$. Now, we use the inequality $c\!\!\mod d \ \leq \ d$, and we have that $\log(a_1) \leq \log(c) - \log(c\!\!\mod d)$. We note that $\size{q+1} = \log(c) + \log(d)$ and that $\size{p+1} = \log(c\!\!\mod d) + \log(d)$ and we obtain that $\log(a_1) \leq \size{q+1} - \size{p+1}$. Lastly, we bring in the inductive assumption for $p$ that $\sum_{i=2}^{n}\log(a_{i}) \leq \size{p+1}$, we add in the previous inequality, and we find that $\sum_{i=1}^{n}\log(a_{i}) \leq \size{q+1}$, which concludes the proof.
\end{proof}

\subsection{Proof of Proposition~\ref{prop:break-links-convergent-equivalence}}

For this result, we utilize some properties of the Stern-Brocot tree that are simply assumed to be known: (1) if $r$ is a value inside the Stern-Brocot tree, and the path from the root to $r$ switches direction on nodes $q_1,\ldots,q_n$, then these are the convergent fractions of $r$; and (2) for any edge $(p,q)$, if $q = \leftChild{p}$, then every node $r$ that satisfies $q < r < p$ is a descendant of $\rightChild{q}$ and vice versa, and if $q = \rightChild{p}$, then every node $r$ that satisfies $p < r < q$ is a descendant of $\leftChild{q}$ and vice versa---and both cases, in turn, imply that $q$ is a convergent of $r$.

To prove part (a) of the statement, fix $r \in \bounds(\gamma)$, and 
let $q \neq 0$ be the some convergent fraction of $r$. Note that if $q = 0$, the statement holds by definition.

We start by looking at the case $q = r$. We will show that either $(q,\leftChild{q})$ or $(q,\rightChild{q})$ is a break link. 
We note that since $r\in \bounds(\gamma)$, then there are two values $r_1,r_2\in\bounds(\gamma)\cup \{-\infty,\infty\}$ such that either $\langle r_1,r)$ and $[r,r_2\rangle$ are maximal monochromatic intervals in $\gamma$, or $\langle r_1,r]$ and $(r,r_2\rangle$ are maximal monochromatic intervals in $\gamma$. We see each case separately:

\begin{enumerate}
    \item  If the intervals are $\langle r_1,r)$ and $[r,r_2\rangle$, then for any sufficiently small $\varepsilon > 0$ it holds that $\gamma(r) \neq \gamma(r - \varepsilon)$. Since one of these satisfies $\varepsilon < r - \leftChild{r}$ then $[\leftChild{r},r]$ is not monochromatic and we see that $(r,\leftChild{r})$ is a break link.
    \item If the intervals are $\langle r_1,r]$ and $(r,r_2\rangle$, then for any sufficiently small $\varepsilon > 0$ it holds that $\gamma(r) \neq \gamma(r + \varepsilon)$. Since one of these satisfies $\varepsilon < \rightChild{r} - r$ then $[r,\rightChild{r}]$ is not monochromatic and we see that $(r,\rightChild{r})$ is a break link.
    
\end{enumerate}

Now we look at the case where $q$ is a strict convergent of $r$. For this case, we 
note that for any strict convergent $q \neq 0$ of $r$, the unordered interval $\unorderInt{q,\parent{q}}$ fully contains the interval $[\leftChild{r}, \rightChild{r}]$.
We recall the two cases described above, where we identified values $r - \varepsilon \in [\leftChild{r},r]$ and $r + \varepsilon \in [r, \rightChild{r}]$, 
for which $\gamma(r) \neq \gamma(r - \varepsilon)$ and $\gamma(r) \neq \gamma(r + \varepsilon)$, respectively. In each case we see that both values are contained in $\unorderInt{q,\parent{q}}$, so we conclude that $(\parent{q},q)$ is a break link, which proves part (a) the statement.

Now, let us prove part (b) of the statement. Let $(p_1,p_2)$ be a break link for $\gamma$. By definition, there are two values $r_1,r_2 \in \unorderInt{p_1,p_2}$ 
such that $\gamma(r_1) \neq \gamma(r_2)$. Assuming $r_1 < r_2$, 
let $\langle r',r\rangle$ be the maximal interval that contains $r_1$ that is monochromatic w.r.t. $\gamma$. 

First, we note that $r\in\bounds(\gamma)$ and that $r$ is a convergent of itself, so if $r = p_1$ or $r = p_2$, then the statement holds. 
We then assume that $r$ is contained strictly inside $\unorderInt{p_1,p_2}$. Again, we refer to the Stern-Brocot tree, which shows that if this holds, then $r$ is reached from the root by a path that goes through $(p_1,p_2)$ and changes direction in $p_2$. We obtain that $p_2$ is a convergent of $r$, which proves part (b) of the statement.

    \section{Proofs of Section 6}

For the following arguments, assume that every time a node $q$ is discovered we also have enough information to compute two values: The value $\lbound{q}$ corresponds to $q_{\ell}$ and the value $\rbound{q}$ corresponds to $q_r$ where $(q_{\ell}, q, q_r)\in B$---i.e., the boundary relation.

\subsection{Proof of Proposition~\ref{prop:hypothesis-construction-correctness}}
Assume that $D$ like ($\ddagger$) contains all break links of $\gamma$. We will prove both statements separately:
\begin{itemize}
\item {\em (1) $q_i$ is a left-bound of $\gamma$ iff $q_{i-1}$ is a descendant of $q_i$ and $A_{i-1} \neq A_{i}$.}
\end{itemize}
    Towards the {\em only if} direction, assume that $q_i$ is a left-bound of $\gamma$.
Then, there exist maximal intervals $\langle p_1, q_i)$ and $[q_i, p_2 \rangle$ that are monochromatic w.r.t. $\gamma$, for some $p_1 \in \bbQ_{-\infty}, p_2 \in \bbQ_{\infty}$.
Consequently, there exists a $\delta > 0$ such that for all $p \in (q_i - \delta, q_i)$ it holds that $\gamma(p)\neq \gamma(q_i) $. Therefore, $(q_i,\leftChild{q_i})$ must be a break link and $ \leftChild{q_i} \leq q_{i-1} < q_i$, so $q_{i-1}$ is a descendant of $q_i$. 
Now, towards a contradiction, assume $A_{i-1} = A_i$. Using $\delta$, we identify a $p \in (q_i - \delta, q_i) \subseteq (q_{i-1}, q_i)$ for which $\gamma(p)\neq A_i$ and so $\gamma(p) \neq A_{i-1}$. 
We note that in the path from $q_{i-1}$ to $p$ there must be a break link that is currently not in $D$, which contradicts the initial assumption and the statement holds.

    For the {\em if direction},  assume that $q_{i-1}$ is a descendant of $q_i$ and that $A_{i-1} \neq A_{i}$, and towards a contradiction, we will assume that $q_i$ is not a left-bound of $\gamma$. 
    This implies that there exists a $\delta > 0$ for which every value $p\in (q_i - \delta, q_i)$ satisfies $\gamma(p) = A_i$.
    Then, there exists a value $q_{i-1} < p < q_i$ such that $\gamma(p)\neq A_{i-1}$. Similarly to above, this implies that $D$ does not contain all break links of $\gamma$, so the statement holds by contradiction.
\begin{itemize}
\item {\em (2) $q_i$ is a right-bound of $\gamma$ iff $q_{i}$ is an ancestor of $q_{i+1}$ and $A_{i+1} \neq A_{i}$.}
\end{itemize}

    Towards the {\em only if} direction, assume that $q_i$ is a right-bound of $\gamma$.
Then, there exist maximal intervals $\langle p_1, q_i]$ and $(q_i, p_2 \rangle$ that are monochromatic w.r.t. $\gamma$, for some $p_1 \in \bbQ_{-\infty}, p_2 \in \bbQ_{\infty}$.
Consequently, there exists a $\delta > 0$ such that for all $p \in (q_i, q_i + \delta)$ it holds that $\gamma(p)\neq \gamma(q_i) $. Therefore, $(q_i,\rightChild{q_i})$ must be a break link and $ q_i < q_{i+1} \leq \rightChild{q_i}$, so $q_{i+1}$ is a descendant of $q_i$. 
Now, towards a contradiction, assume $A_{i+1} = A_i$. Using $\delta$, we identify a $p \in (q_i, q_i+\delta) \subseteq (q_i, q_{i+1})$ for which $\gamma(p)\neq A_i$ and so $\gamma(p) \neq A_{i+1}$. 
We note that in the path from $q_{i+1}$ to $p$ there must be a break link that is currently not in $D$, which contradicts the initial assumption and the statement holds.

    For the {\em if direction},  assume that $q_{i+1}$ is a descendant of $q_i$ and that $A_{i+1} \neq A_{i}$, and towards a contradiction, we will assume that $q_i$ is not a right-bound of $\gamma$. 
    This implies that there exists a $\delta > 0$ for which every value $p\in (q_i, q_i+\delta)$ satisfies $\gamma(p) = A_i$.
    Then, there exists a value $q_{i} < p < q_{i+1}$ such that $\gamma(p)\neq A_{i+1}$. Similarly to above, this implies that $D$ does not contain all break links of $\gamma$, so the statement holds by contradiction.

\begin{lemma}
At every iteration of Algorithm \ref{alg:learning-algo}, for any element $(q_i, \gamma(q_i)) \in D$ for $1 \leq i \leq$ with $|D| = m$, $q_{i-1}$ and $q_{i+1}$ are connected to $q_i$ by a descendant or ancestor relationship.
\end{lemma}

\begin{proof}
We will prove that the invariant holds for every $j$-th iteration of the algorithm using induction.

First, before the first iteration of the algorithm, the property trivially holds because $D$ contains only $(\frac{0}{1}, \gamma(\frac{0}{1}))$ ($|D| = 1$).

Assume that the property holds for the $j$-th iteration of the algorithm. We will prove that the property holds for the $(j+1)$-th iteration of the algorithm.

Let $D$ be the following in the $j$-th iteration of the algorithm:

\[
D = (q_1,A_1), \ldots, (q_{i-1}, A_{i-1}), (q_i, A_i), (q_{i+1}, A_{i+1}), \ldots (q_{m}, A_m)
\]

In Algorithm~\ref{alg:learning-algo}, at the beginning of the iteration, in Line~\ref{in-alg:ask-equivalence}, we receive a counterexample $q^{*}$. Let $q_i$ be the closest ancestor of $q^{*}$ which we obtained in Line~\ref{in-alg:find-closest-ancestor}. Cause Algorithm~\ref{alg:find-break-link}, called in Line~\ref{in-alg:find-break-link}, find a break link $(p,q)$ in the left or right-branch of $q_i$, then $p$ and $q$ must be descendants of $q_i$.

Now, we will prove that the invariant holds for every iteration of the algorithm:

\begin{itemize}
    \item If $q^{*} < q_i$: the result array $D$ will have the following elements:
    \[
    (q_{i-1}, A_{i-1}), (q, \gamma(q)), (p, \gamma(p)), (q_i, A_i)
    \]
    where $\searchLeftName$ returns a break link $(p,q)$ such that $q_{i-1} \leq q < p \leq q_i$.

    By the induction hypothesis:
    
    \begin{itemize}
        \item If $q_{i-1}$ was a descendant of $q_i$: $p$ and $q$ are in the left-branch of $q_i$ and $q \geq q_{i-1}, p > q_{i-1}$ then $p$ and $q$ must be ascendants of $q_{i-1}$.

        \item If $q_{i-1}$ was an ascendant of $q_i$: $p$ and $q$ are in the left-branch of $q_i$ so $p$ and $q$ are descendants of $q_{i}$, by the transitivity property, then $p$ and $q$ must be descendants of $q_{i-1}$.
    \end{itemize}

    $p$ is an ascendant of $q$, because they form part of an edge in $\SBtree$.

    Finally, $q$ is an ascendant or descendant of $q_{i-1}$ in any case, $p$ is an ascendant of $q$ and $p$ is a descendant of $q_i$. Therefore, the property holds.

    \item If $q^{*} > q_i$: the result array $D$ will have the following elements:
    \[
    (q_i, A_i), (p, \gamma(p)), (q, \gamma(q)), (q_{i+1}, A_{i+1})
    \]
    where $\searchRightName$ returns a break link $(p,q)$ such that $q_{i-1} \leq q < p \leq q_i$.

    By the induction hypothesis:
    
    \begin{itemize}
        \item If $q_{i+1}$ was a descendant of $q_i$: $p$ and $q$ are in the right-branch of $q_i$ and $q \leq q_{i+1}, p < q_{i+1}$ then $p$ and $q$ must be ascendants of $q_{i+1}$.

        \item If $q_{i+1}$ was an ascendant of $q_i$: $p$ and $q$ are in the right-branch of $q_i$ so $p$ and $q$ are descendants of $q_{i}$, by the transitivity property, then $p$ and $q$ must be descendants of $q_{i+1}$.
    \end{itemize}

    $p$ is an ascendant of $q$, because they form part of an edge in $\SBtree$.

    Finally, $q$ is an ascendant or descendant of $q_{i+1}$ in any case, $p$ is an ascendant of $q$ and $p$ is a descendant of $q_i$. Therefore, the property holds.

\end{itemize}

We proved by induction that the property holds for every iteration $j$-th in Algorithm~\ref{alg:learning-algo}.
\end{proof}

\subsection{Proof of Theorem~\ref{theo:hip-construction}}

First, we will prove that Algorithm~\ref{alg:hip-constructions} always outputs a consistent representation $\rep$ of $D$. We will prove that every element $(q,b)$ in $D$ holds $\gamma(q) = \rep(q)$.

We will divide the proof into two parts. First, we will prove the property for every $q$ that is left or right-bound in $\rep$. Second, we will prove the property for the rest of elements in $D$.

First, Let $q$ be any left or right-bound of $\rep$. In the construction of $\rep$ using Algorithm~\ref{alg:hip-constructions}, $q$ must have satisfied conditions on Lines 5 or 8. If Line~5 was fulfilled, the element $([q, p \rangle, \gamma(q))$ will belong to $\rep$, then $\rep(q) = \gamma(q)$. If Line~8 was fulfilled, the element $(\langle p, q], \gamma(q))$ will belong to $\rep$, then $\rep(q) = \gamma(q)$.

Second, let $(\langle p, r \rangle, A)$ be the interval in $\rep$ that $q$ belongs to where $\langle \in \{(,[\}$, $\rangle \in \{),]\}$ and $p \leq r$. Let the number of elements of $D$ between $p$ and $q$ be of size $m$.

We will prove it considering the four types of intervals:

\begin{itemize}
    \item $[p, p]$: Cause the interval is a singleton and $p$ is the unique endpoint, we proved before that $\gamma(p) = \rep(p).$

    For the next tree cases, assume that $D$ contains at least two elements in the interval. If only one or none exists, then it holds trivially by the endpoint argument. By Lemma~1, every element in $D$ has an ascendant or descendant relationship with the elements next to it.
    Let the elements in $D$ that belong to the interval $(p,r)$ be of the form: $p_1, \ldots, p_m$. We will prove by induction on $m$:
    
    \item $[p, r)$: For the base case ($m=1$): By contradiction, assume that $\gamma(p_1) \neq \gamma(p)$, then: ($1$) If $p_1$ is a descendant of $p$ then, by Line~9 in $\constructRepresentationName$, $p$ would have been right-bound of an interval in $\rep$. ($2$) If $p$ is a descendant of $p_1$, then the interval $[p, r)$ would have been partitioned exactly on $p_1$ with $p_1 < r$. Then $[p, r)$ could not have been in $\rep$. Therefore, $p_1$ must hold that $\gamma(p_1) = \gamma(p)$. We assume the property holds for $m=i$. For $m=i+1$: By contradiction, assume that $\gamma(p_i) \neq \gamma(p_{i+1})$, then: If $p_{i+1}$ is a descendant of $p_i$ or $p_i$ is a descendant of $p_{i
    +1}$ then, the interval $[p, r)$ would have been partitioned exactly on $p_i$ or $p_{i+1}$ with $p_i, p_{i+1} < r$. Then $[p, r)$ could not have been in $\rep$. Therefore, $p_{i+1}$ must hold that $\gamma(p_{i+1}) = \gamma(p_i)$. Finally, by induction, the property holds for every element in $D$ that belongs to $(p,r)$.

    \item $(p, r]$: For the base case ($m=1$): By contradiction, assume that $\gamma(p_1) \neq \gamma(r)$, then: ($1$) If $p_1$ is a descendant of $r$ then, by Line~6 in $\constructRepresentationName$, $r$ would have been left-bound of an interval in $\rep$. ($2$) If $r$ is a descendant of $p_1$, then the interval $(p, r]$ would have been partitioned exactly on $p_1$ with $p_1 < r$. Then $(p, r]$ could not have been in $\rep$. Therefore, $p_1$ must hold that $\gamma(p_1) = \gamma(r)$. We assume the property holds for $m=i$. For $m=i+1$: By contradiction, assume that $\gamma(p_i) \neq \gamma(p_{i+1})$, then: If $p_{i+1}$ is a descendant of $p_i$ or $p_i$ is a descendant of $p_{i
    +1}$ then, the interval $(p, r]$ would have been partitioned exactly on $p_i$ or $p_{i+1}$ with $p_i, p_{i+1} < r$. Then $(p, r]$ could not have been in $\rep$. Therefore, $p_{i+1}$ must hold that $\gamma(p_{i+1}) = \gamma(p_i)$. Finally, by induction, the property holds for every element in $D$ that belongs to $(p,r)$.

    \item $(p, r)$: By the construction of $\rep$, $p_1$ determined the evaluation of the interval in $\rep$. For the base case ($m=1$): it holds the property trivially. We assume the property holds for $m=i$. For $m=i+1$: By contradiction, assume that $\gamma(p_i) \neq \gamma(p_{i+1})$, then: If $p_{i+1}$ is a descendant of $p_i$ or $p_i$ is a descendant of $p_{i
    +1}$ then, the interval $(p, r)$ would have been partitioned exactly on $p_i$ or $p_{i+1}$ with $p_i, p_{i+1} < r$. Then $(p, r)$ could not have been in $\rep$. Therefore, $p_{i+1}$ must hold that $\gamma(p_{i+1}) = \gamma(p_i)$. Finally, by induction, the property holds for every element in $D$ that belongs to $(p,r)$.
\end{itemize}

Therefore, Algorithm~\ref{alg:hip-constructions} always output a consistent representation $\rep$ of $D$.

Now, we want to prove that if $D$ contains all break links of~$\gamma$, then $\rep = \gamma$.

If $D$ contains all break links of~$\gamma$, then for every edge $(p,q) \in \SBtree$ such that $p,q \notin D$, then $\unorderInt{p,q}$ must be monochromatic. As a result, every interval in $\rep$ is monochromatic on~$\gamma$ and, by Proposition~\ref{prop:break-links-convergent-equivalence} we have all the convergent fractions of $\bounds{\gamma}$, then $\gamma = \rep$.

\subsection{Proof of Proposition~\ref{prop:find-breaklink}}
By the definition of the equivalence oracle in Section~3, it holds that $\gamma(q^{*}) = \rep(q^{*})$.

Let $D$ be the following in the current iteration of Algorithm~\ref{alg:learning-algo}:
\[
D = (q_1,A_1, \ldots, (q_{i-1}, A_{i-1})), (q_i, A_i), (q_{i+1}, A_{i+1}), \ldots (q_{m}, A_m)
\]

We will prove that if $q^{*} < r$ there exists a break link $(p,q)$ in the left-branch of $r$ where $(q, \gamma(q)) \notin D$ and if $q^{*} > r$ there exists a break link $(p,q)$ in the right-branch of $r$ where $(q, \gamma(q)) \notin D$.

First, if $q^{*} < r$ we will prove that there exists a break link $(p,q)$ in the left-branch of $r$ where $(q, \gamma(q)) \notin D$. We will prove it by cases:

\begin{itemize}
    \item If $q_{i-1}$ is not a descendant of $q_i$. Cause $q_j < q_{i-1}. \forall 1 \leq j < i-1$ by the construction of $D$, then no break link has been discovered in the left-subtree of $q_i$ until this iteration of the Algorithm. In this case, $\rep((\lbound{q_i}, q_i)) = A$ for some label $A$. Because, $\gamma(q^{*}) \neq \rep(q^{*})$, then in the left-subtree of $q_i$ it must exist a bound that partitions $(\lbound{q_i}, q_i)$ so a break link must exist in the left-branch of $q_i$.

    \item If $q_{i-1}$ is a descendant of $q_i$. It holds that $\rbound{q_{i-1}} < q^{*} < q_{i}$, if not $q_{i-1}$ would be the closest ancestor of $q_{*}$. In this case, $\rep([\rbound{q_{i-1}}, q_i)) = A$ for some label $A$. Because, $\gamma(q^{*}) \neq \rep(q^{*})$, then in the left-subtree of $q_i$ it must exist a bound that partitions $[\rbound{q_{i-1}}, q_i)$ so a break link must exist in the left-branch of $q_i$ in $[\rbound{q_{i-1}}, q_i)$.
\end{itemize}

Second, if $q^{*} > r$ we will prove that there exists a break link $(p,q)$ in the right-branch of $r$ where $(q, \gamma(q)) \notin D$. We will prove it by cases:

\begin{itemize}
    \item If $q_{i+1}$ is not a descendant of $q_i$. Cause $q_j > q_{i+1}. \forall i+1 < j \leq m$ by the construction of $D$, then no break link has been discovered in the right-subtree of $q_i$ until this iteration of the Algorithm. In this case, $\rep((q_i, \rbound{q_i})) = A$ for some label $A$. Because, $\gamma(q^{*}) \neq \rep(q^{*})$, then in the right-subtree of $q_i$ it must exist a bound that partitions $(q_i, \rbound{q_i})$ so a break link must exist in the right-branch of $q_i$.

    \item If $q_{i+1}$ is a descendant of $q_i$. It holds that $q_i < q^{*} < \lbound{q_{i+1}}$, if not $q_{i+1}$ would be the closest ancestor of $q_{*}$. In this case, $\rep((q_i, \lbound{q_{i+1}}]) = A$ for some label $A$. Because, $\gamma(q^{*}) \neq \rep(q^{*})$, then in the right-subtree of $q_i$ it must exist a bound that partitions $(q_i, \lbound{q_{i+1}}]$ so a break link must exist in the right-branch of $q_i$ in $(q_i, \lbound{q_{i+1}}]$.
\end{itemize}

\subsection{Proof of Theorem~\ref{theo:find-breaklink}}
Let $D$ be the following in the current iteration of Algorithm~\ref{alg:learning-algo}:
\[
D = (q_1,A_1), \ldots, (q_{i-1}, A_{i-1}), (q_i, A_i), (q_{i+1}, A_{i+1}), \ldots ,(q_{m}, A_m)
\]

Given a counterexample $q^{*}$ and $r$ the closest ancestor of $q^{*}$ in $D$. If $q^{*} > r$, then Algorithm~$\searchRightName$ is called. We will prove that it returns a break link $(p,q)$ not yet present in $D$. Let $q_i$ be equal to $r$ in $D$. Note that $q_i < q^* < \rbound{q_i}$.

There are four cases, we will prove that $\searchRightName$ finds an undiscovered break link for each case separately:

\begin{itemize}
    \item {\em (1) $q_{i+1}$ is an ancestor of $q_i$ and $A_{i+1} = A_i$}
    
    In this case, there does not exist any yet-discovered break link in the right-branch of $q_i$ and $\rep([q_i, q_{i+1}]) = A_i$ as $\rep$ is consistent with $D$ as demonstrated in Theorem~\ref{theo:hip-construction}. Therefore, $\gamma(q^{*}) \neq A_i$ as $q^{*} \in (q_i, q_{i+1})$ and $\gamma(q^{*}) \neq \rep(q^{*})$ by the definition of the Equivalence Oracle.

    Suppose we are at the $k$-th iteration (starting with the 0-th one) on the loop in Algorithm~$\searchRightName$. We denote by $e_k$ the node computed in Line~\ref{searchRight:compute-e}.
    
Since $e_k$ can become arbitrarily close to $\rbound{q_i}$, then at some point it must hold that $e_k > q^*$. Since $\gamma(q^{*}) \neq A_i$, then by the condition in Line~\ref{searchRight:cond2}, we deduce that the loop always stops.

    If condition in Line~\ref{searchRight:cond1} is at some point fulfilled, then $\gamma(e_{k-1}) \neq \gamma(e_k)$ for some $k$, then there is an undiscovered break link between these nodes.
    Otherwise, and if instead it is the one on Line~\ref{searchRight:cond2} that is fulfilled, then that means that for some $k$, it holds that $\gamma(e_{k-1}) = \gamma(e_{k})$ and $q^*\in (e_{k-1}, e_k)$ has a different color. Therefore, there is an undiscovered break link between $e_{k-1}$ and $e_k$.
    The condition in Line~\ref{searchRight:cond3} is never  fulfilled because $e_k < q_{i+1}$ for every $k$.

    \item {\em (2) $q_{i+1}$ is an ancestor of $q_i$ and $A_{i+1} \neq A_i$}
    
    In this case, there does not exist any yet-discovered break link in the right-branch of $q_i$ and $\rep([q_i, q_{i+1})) = A_i$ as $\rep$ is consistent with $D$ as demonstrated in Theorem~\ref{theo:hip-construction} and how $\rep$ is constructed by $\constructRepresentationName$. Therefore, $\gamma(q^{*}) \neq A_i$ as $q^{*} \in (q_i, q_{i+1})$ and $\gamma(q^{*}) \neq \rep(q^{*})$ by the definition of the Equivalence Oracle. The rest of the proof for this case is analogous to item (1) above.

    \item {\em (3) $q_{i+1}$ is a descendant of $q_i$ and $A_{i+1} = A_i$}
    
    In this case, $\rep([q_i, q_{i+1}]) = A_i$ as $\rep$ is consistent with $D$ as demonstrated in Theorem~\ref{theo:hip-construction} and how $\rep$ is constructed by $\constructRepresentationName$. Therefore, $\gamma(q^{*}) \neq A_i$ as $q^{*} \in (q_i, q_{i+1})$ and $\gamma(q^{*}) \neq \rep(q^{*})$ by the definition of the Equivalence Oracle.

    Suppose we are at the $k$-th iteration (starting with the 0-th one) on the loop in Algorithm~$\searchRightName$, we denote by $e_k$ the node reached at iteration $k$.
    
    Since $e_k$ can become arbitrarily close to $\rbound{q_i}$, then at some point it must hold that $e_k > q^*$. Since $\gamma(q^{*}) \neq A_i$, then by the condition in Line~\ref{searchRight:cond2}, we deduce that the loop always stops.
    
    If condition in Line~\ref{searchRight:cond1} is at some point fulfilled, then $\gamma(e_{k-1}) \neq \gamma(e_k)$ for some $k$, then there is an undiscovered break link between these nodes.
    Otherwise, and if instead it is the one on Line~\ref{searchRight:cond2} that is fulfilled, then that means that for some $k$, it holds that $\gamma(e_{k-1}) = \gamma(e_{k})$ and $q^*\in (e_{k-1}, e_k)$ has a different color. Therefore, there is an undiscovered break link between $e_{k-1}$ and $e_k$.
    
    If the condition on Line~\ref{searchRight:cond3} is fulfilled before the one on Line~\ref{searchRight:cond2}, and since $\gamma(q^{*}) \neq A_i$, then this implies that $q^* \leq q_{i+1}$ which contradicts the fact that $q_i$ is the closest ancestor to $q^*$.

    \item {\em (4) $q_{i+1}$ is a descendant of $q_i$ and $A_{i+1} \neq A_i$}
    
    In this case, $\rep((q_i, q_{i+1}]) = A_{i+1}$ as $\rep$ is consistent with $D$ as demonstrated in Theorem~\ref{theo:hip-construction} and how $\rep$ is constructed by $\constructRepresentationName$. Therefore, $\gamma(q^{*}) \neq A_{i+1}$ as $q^{*} \in (q_i, q_{i+1})$ and $\gamma(q^{*}) \neq \rep(q^{*})$ by the definition of the Equivalence Oracle. Note that in this case, it could happen that $\gamma(q^{*}) = A_i$.

    Suppose we are at the $k$-th iteration (starting with the 0-th one) on the loop in Algorithm~$\searchRightName$, we denote by $e_k$ the node reached at iteration $k$.
    
    We have that $q_{i+1} < \rbound{q_i}$, and since $e_k$ can become arbitrarily close to $\rbound{q_i}$, by the condition in Line~\ref{searchRight:cond3}, we deduce that the loop always stops.

    If condition in Line~9 is fulfilled, then cause $\gamma(e_{k-1}) \neq \gamma(e_k)$ a break link must exist between these two nodes. 
    Otherwise, and if instead it is the one on Line~\ref{searchRight:cond2} that is fulfilled, then that means that for some $k$, it holds that $\gamma(e_{k-1}) = \gamma(e_{k})$ and $q^*\in (e_{k-1}, e_k)$ has a different color. Therefore, there is an undiscovered break link between $e_{k-1}$ and $e_k$.
    On the other hand, if it is the condition on Line~\ref{searchRight:cond3} that is fulfilled, then that means that for some $k$, it holds that $\gamma(e_{k-1}) = \gamma(e_{k})$ and $q_{i+1} \in (e_{k-1}, e_k)$ has a different color. Similarly, there is an undiscovered break link between $e_{k-1}$ and $e_k$.
    
\end{itemize}

Now, to prove that $\bsRightName$ returns an undiscovered break link in the sub-right-branch of $q_i$. We will have as invariant that at each iteration, there is an undiscovered break link between nodes $q_{\sf lo} = \frac{a + {\sf lo}\times c}{b + {\sf lo}\times d}$ and $q_{\sf hi} = \frac{a + {\sf hi}\times c}{b + {\sf hi}\times d}$. As a base case, we note that before entering the loop, these nodes correspond to nodes $e_{k-1}$ and $e_k$ described above, which satisfy the invariant. 
Note that this immediately implies that if $q_{\sf hi}$ is ever child of $q_{\sf lo}$ then the edge $(q_{\sf lo}, q_{\sf hi})$ is a break link, which is what is checked in Line~\ref{bsRight:cond1}. 
Similarly, in Lines~\ref{bsRight:cond3} and~\ref{bsRight:cond4}, the conditions describe that the edge is a break link.

If condition in Line~\ref{bsRight:cond5} is fulfilled, then $\gamma(e_{k-1}) \neq \gamma(p)$, therefore between $p$ and $e_{k}$ a break link must exist. In other case, then $\gamma(e_{k-1}) = \gamma(p)$, then $q^{*}$ determines if the break link is in the top part of the array or lower part. As the array get divided conserving the break link, by the induction hypothesis, $\bsRightName$ returns a break link $(p,q)$ in the right-branch of $q_i$ between $q_i$ and $q_{i+1}$.

The fact that this break link is so-far undiscovered is guaranteed by the fact that these nodes are in the right branch of $q_i$; if $q_{i+1}$ is ancestor of $q_i$, then they are strictly less than $q_i$; if $q_{i+1}$ is descendant of $q_i$, then the only way this break link is already seen is if at the end of the binary search, then $q_i = q_{\sf lo}$ and $q_{i+1} = q_{\sf hi}$ and this means that $q_{i+1}$ is child of $q_i$, and since we have that $q_i < q^* < q_{i+1}$, then this implies that $q^*$ is a descendant of $q_{i+1}$, which contradicts the fact that $q_i$ is the closest ancestor to $q^*$.

For Algorithm~$\searchLeftName$, the reasoning is fully analogous.

Regardless $\searchRightName$ or $\searchLeftName$ is called, the loop defined in Line~5 in either algorithm is executed exactly $\lceil \log(k) \rceil$ times, since: $2^{\lceil \log(k) \rceil} \geq k$. The first procedure of the algorithm takes $\bigo(\log(k))$ time. Because, binary search takes $\bigo(n)$ time, where $n$ is the size of the array, then the interval searched is $2^{\log(k)} - 2^{\log(k-1)} = 2^{\log(k-1)}$. The second procedure then takes $\bigo(\log(2^{\log(k-1)})) = \bigo(\log(k))$.

\subsection{Proof of Theorem~\ref{theo:correctness-learning-algo}}
Algorithm~\ref{alg:learning-algo} returns a representation $\rep$ for a target function $\gamma \in \cC$ using no more than $\bigo(\size{\gamma})$ membership queries and no more than $\bigo(\size{\gamma})$ equivalence queries.

{\em Proof.}
As Algorithm~\ref{alg:learning-algo}, at every iteration, discovers a new break link and store in $D$ and the number of break links is finite by Proposition~\ref{prop:break-links-convergent-equivalence}. At some iteration, $D$ will contain every break link of~$\gamma$ and Algorithm~\ref{alg:hip-constructions} will return $\rep$ such that $\rep=\gamma$ as demonstrated in Theorem~\ref{theo:hip-construction}.

In Section~\ref{sec:algorithm} we analyze the number of membership queries and equivalence queries needed to learn an unknown finite piecewise function $\gamma$ using Algorithm~\ref{alg:learning-algo}. As proved, the Algorithm~\ref{alg:learning-algo} will use no more than $\bigo(\size{\gamma})$ membership queries and no more than $\bigo(\size{\gamma})$ equivalence queries.

---

\begin{algorithm}[H]
    \caption{$\searchRightName$: Return a break link on the right-branch of a node $r$}
    \label{alg:search-right}
    \begin{algorithmic}[1]
    \REQUIRE The counterexample $q^{*} \in \bbQ$, the closest ancestor $r$ of $q^{*}$ in $D$ and array $D$\\
    \ENSURE A break link $(p,q)$ in the right-branch that stems from $r$ \vspace{1mm}
	\STATE {\bf Procedure}  $\searchRight{r}{q^{*}}{D}$
	\STATE {\bf assume} $D = (q_1, A_1), (q_2, A_2), \ldots, (q_m, A_m)$ {\bf with } $r = q_i$ \COMMENT{Note: $\gamma(r) = A_i$}
    \STATE $\tfrac{c}{d} \gets \rbound{r}$
    \STATE $r_{\sf right} \gets q_{i+1}$ {\bf if } $i < m$ {\bf otherwise} $r_{\sf right} \gets +\infty$ 
    \STATE $k \gets 0$
    \LOOP
        \STATE $e \gets \tfrac{a+2^{k}\cdot c}{b + 2^{k}\cdot d}$ \label{searchRight:compute-e}\vspace{1mm}
        \STATE $eLabel \gets \membOracle{e}$
        \IF {$A_i \neq \membOracle{e}$} \label{searchRight:cond1}
            \STATE \textbf{return} $\bsRight{r}{q^{*}}{k}{r_{\sf right}}$
        \ELSIF {($\membOracle{q^{*}} \neq A_i$) {\bf and} ($e > q^{*}$)} \label{searchRight:cond2}
            \STATE \textbf{return} $\bsRight{r}{q^{*}}{k}{r_{\sf right}}$ 
        \ELSIF{$e > r_{\sf right}$} \label{searchRight:cond3}
            \STATE \textbf{return} $\bsRight{r}{q^{*}}{k}{r_{\sf right}}$
        \ENDIF
        \STATE $k \gets k +1$
    \ENDLOOP
    \end{algorithmic}
\end{algorithm}

\begin{algorithm}[H]
    \caption{$\bsRightName$: Return a break link in the bounded right-branch of a node $r$}
    \label{alg:binary-search-right}
    \begin{algorithmic}[1]
    \REQUIRE The counterexample $q^{*} \in \bbQ$, the closest ancestor $r = \frac{a}{b}$, the integer $k \in \bbN$ of the number of exponential jumps and the next break link after $r$ denoted by $r_{\sf right}$. \\
    \ENSURE A break link $(p,q)$ in the right-branch of $r$. \\
    \STATE {\bf Procedure}  $\bsRight{r}{q^{*}}{k}{r_{\sf right}}$
    \IF{k = 0}
    	\STATE {\bf return} $(r,\rightChild{r})$
    \ENDIF
    \STATE $\frac{c}{d} \gets \rbound{r}$
    \STATE ${\sf lo} \gets 2^{k-1}$, ${\sf hi} \gets 2^{k}$
    \STATE $rLabel \gets \membOracle{r}$
    \LOOP
        \STATE ${\sf mid} \gets \lceil \frac{{\sf lo}+{\sf hi}}{2} \rceil$ \vspace{1mm}
        \STATE $q \gets \frac{a + {\sf mid}\times c}{b + {\sf mid}\times d}$ \vspace{1mm},  $qLabel \gets \membOracle{q}$\vspace{1mm}
        \STATE $p \gets \parent{q}$\vspace{1mm},  $pLabel \gets \membOracle{p}$
        \IF {${\sf hi} - {\sf lo} \leq 2$}  \label{bsRight:cond1}

            \STATE \textbf{return} $(p,q)$
        \ELSIF {$q \leq r_{\sf right}$ }  \label{bsRight:cond2}

            \STATE ${\sf hi} \gets {\sf mid} - 1$
        \ELSE
            \IF {$pLabel \neq qLabel$} \label{bsRight:cond3}
                \STATE \textbf{return} $(p,q)$
            \ENDIF
            \IF {$rLabel = qLabel$ {\bf and} $p < q^{*} < q$ {\bf and} $\membOracle{q^{*}} \neq rLabel$} \label{bsRight:cond4}

                \STATE \textbf{return} $(p,q)$
            \ENDIF
            \IF {$rLabel \neq pLabel$ } \label{bsRight:cond5}

                \STATE ${\sf hi} \gets {\sf mid} - 1$
            \ENDIF
            \IF {$\membOracle{q^{*}} \neq rLabel$} \label{bsRight:cond6}

                \IF {$q^{*} > p$ }  \label{bsRight:cond7}

                    \STATE ${\sf lo} \gets {\sf mid} + 1$
                \ELSE
                    \STATE ${\sf hi} \gets {\sf mid} - 1$
                \ENDIF
            \ELSE
                \STATE ${\sf lo} \gets {\sf mid} + 1$
            \ENDIF
        \ENDIF
    \ENDLOOP
    \end{algorithmic}
\end{algorithm}

\begin{algorithm}[H]
    \caption{$\searchLeftName$: Return a break link on the left branch of a node $r$}
    \label{alg:left-right}
    \begin{algorithmic}[1]
    \REQUIRE The counterexample $q^{*} \in \bbQ$, the closest ancestor $r$ of $q^{*}$ in $D$ and array $D$\\
    \ENSURE A break link $(p,q)$ in the left-branch that stems from $r$ \vspace{1mm}
    \STATE {\bf Procedure}  $\searchLeft{r}{q^{*}}{D}$
    	\STATE {\bf assume} $D = (q_1, A_1), (q_2, A_2), \ldots, (q_m, A_m)$ {\bf with } $r = q_i$ \COMMENT{Note: $\gamma(r) = A_i$}
    \STATE $\tfrac{c}{d} \gets \lbound{r}$
    \STATE $r_{\sf left} \gets q_{i-1}$  {\bf if } $i > 1$ {\bf otherwise} $r_{\sf left} \gets -\infty$ 
    \STATE $k \gets 0$
    \LOOP
        \STATE $e \gets \tfrac{a+2^{k}\cdot c}{b + 2^{k}\cdot d}$ \vspace{1mm}
        \STATE $eLabel \gets \membOracle{e}$
        \IF {$\membOracle{e} = A_i$}
            \STATE \textbf{return} $\bsLeft{r}{q^{*}}{k}{r_{\sf left}}$
        \ENDIF
        \IF {$\membOracle{q^{*}} \neq A_i$ {\bf and} $e < q^{*}$}
            \STATE \textbf{return} $\bsLeft{r}{q^{*}}{k}{r_{\sf left}}$
        \ELSIF{$e < r_{\sf left}$}
            \STATE \textbf{return} $\bsLeft{r}{q^{*}}{k}{r_{\sf left}}$
        \ENDIF
        \STATE $k \gets k +1$
    \ENDLOOP
    \end{algorithmic}
\end{algorithm}

\begin{algorithm}[H]
    \caption{$\bsLeftName$: Return a break link in a sub-branch of $\SBtree$}
    \label{alg:binary-search-left}
    \begin{algorithmic}[1]
    \REQUIRE The counterexample $q^{*} \in \bbQ$, the closest ancestor $r = \frac{a}{b}$, the integer $k \in \bbN$ of the number of exponential jumps and the next break link before $r$ denoted by $r_{\sf left}$. \\
    \ENSURE A break link $(p,q)$ in the left-branch of $r$. \\
    \STATE {\bf Procedure}  $\bsLeft{r}{q^{*}}{k}{r_{\sf left}}$
    \IF{k = 0}
    	\STATE {\bf return} $(r,\leftChild{r})$
    \ENDIF
    \STATE $\tfrac{c}{d} \gets \lbound{r}$
    \STATE ${\sf lo} \gets 2^{k-1}$, ${\sf hi} \gets 2^{k}$
    \STATE $rLabel \gets \membOracle{r}$
    \LOOP
        \STATE ${\sf mid} \gets \lceil \frac{{\sf lo}+{\sf hi}}{2} \rceil$ \vspace{1mm}
        \STATE $q \gets \frac{a + {\sf mid}\times c}{b + {\sf mid}\times d}$ \vspace{1mm}, $qLabel \gets \membOracle{q}$
        \STATE $p \gets \parent{q}$,  $pLabel \gets \membOracle{p}$
        \IF {${\sf hi} - {\sf lo} \leq 2$}
            \STATE \textbf{return} $(p,q)$
        \ENDIF
        \IF {$q \geq r_{\sf left}$}
            \STATE ${\sf hi} \gets {\sf mid} - 1$
        \ELSE
            \IF {$pLabel \neq qLabel$}
                \STATE \textbf{return} $(p,q)$
            \ENDIF
            \IF {$rLabel = qLabel$ {\bf and} $q < q^{*} < p$ {\bf and} $\membOracle{q^{*}} \neq rLabel$}
                \STATE \textbf{return} $(p,q)$
            \ENDIF
            \IF {$rLabel \neq pLabel$ }
                \STATE ${\sf hi} \gets {\sf mid} - 1$
            \ENDIF
            \IF {$\membOracle{q^{*}} \neq rLabel$}
                \IF {$q^{*} > p$ }
                    \STATE ${\sf hi} \gets {\sf mid} - 1$
                \ELSE
                    \STATE ${\sf lo} \gets {\sf mid} + 1$
                \ENDIF
            \ELSE
                \STATE ${\sf lo} \gets {\sf mid} + 1$
            \ENDIF
        \ENDIF
    \ENDLOOP
    \end{algorithmic}
\end{algorithm}

\end{document}